\documentclass{bmvc2k}

\usepackage{bmvc2k_natbib}

\usepackage{microtype}
\usepackage{graphicx}
\usepackage{booktabs} 
\usepackage{hyperref}

\usepackage{amsmath}
\usepackage{mathtools}
\usepackage{amsthm}
\usepackage{dsfont}

\usepackage{xcolor, booktabs, multirow, array, placeins, pifont, enumitem, bm, soul, nicefrac, bbold, boldline, wrapfig, graphicx, comment, xspace, colortbl}

\usepackage{anyfontsize}

\definecolor{ref-color}{HTML}{C90016}

\definecolor{eq-fig-tab-color}{HTML}{D70040}

\definecolor{third-best}{HTML}{138808}
\definecolor{sec-best}{HTML}{002FA7}
\definecolor{best}{HTML}{FF033E}

\definecolor{col-color}{HTML}{C9C0BB}

\definecolor{zvad-color}{rgb}{0.886274509803922,0.290196078431373,0.2}
\definecolor{other-color}{rgb}{0.203921568627451,0.541176470588235,0.741176470588235}

\definecolor{ped1-color}{rgb}{0.886274509803922,0.290196078431373,0.2}
\definecolor{ped2-color}{rgb}{0.203921568627451,0.541176470588235,0.741176470588235}
\definecolor{ave-color}{rgb}{0.29, 0.33, 0.13}

\definecolor{bar-color}{HTML}{273BE2}

\definecolor{checkmark}{HTML}{40826D}
\definecolor{xmark}{HTML}{E62020}

\makeatletter
\def\dual#1{\expandafter\dual@aux#1\@nil}
\def\dual@aux#1/#2\@nil{\begin{tabular}{@{}c@{}}#1\\#2\end{tabular}}
\makeatother

\makeatletter
\newcommand\sfsize{\@setfontsize\sfsize\@viipt\@viiipt}
\makeatother

\makeatletter
\renewcommand{\paragraph}{%
  \vspace{-.8em} 
  \@startsection{paragraph}{4}%
  {\z@}{2ex \@plus 1ex \@minus .2ex}{-0.5em}%
  {\normalfont\normalsize\bfseries}%
}
\makeatother

\definecolor{OursColor}{rgb}{1.0,0.88,0.88}

\usepackage{arydshln}
\makeatletter
\def\adl@drawiv#1#2#3{%
        \hskip.5\tabcolsep
        \xleaders#3{#2.5\@tempdimb #1{1}#2.5\@tempdimb}%
                #2\z@ plus1fil minus1fil\relax
        \hskip.5\tabcolsep}
\newcommand{\cdashlinelr}[1]{%
  \noalign{\vskip\aboverulesep
           \global\let\@dashdrawstore\adl@draw
           \global\let\adl@draw\adl@drawiv}
  \cdashline{#1}
  \noalign{\global\let\adl@draw\@dashdrawstore
           \vskip\belowrulesep}}
\makeatother

\definecolor{darkpastelgreen}{rgb}{0.0, 0.5, 0.0}
\definecolor{britishracinggreen}{rgb}{0.0, 0.26, 0.15}

\definecolor{blush}{rgb}{0.87, 0.36, 0.51}

\newcommand{\ourso}{ContCore (\textbf{ours})\xspace}
\newcommand{\ours}{ContCore\@\xspace}

\makeatletter
\DeclareRobustCommand\onedot{\futurelet\@let@token\@onedot}
\def\@onedot{\ifx\@let@token.\else.\null\fi\xspace}

\def\eg{\emph{e.g}\onedot} 
\def\ie{\emph{i.e}\onedot}

\newcommand*\bigcdot{\mathpalette\bigcdot@{.5}}
\newcommand*\bigcdot@[2]{\mathbin{\vcenter{\hbox{\scalebox{#2}{$\m@th#1\bullet$}}}}}
\makeatother

\RequirePackage{amsmath}
\RequirePackage{amssymb}
\RequirePackage{amsthm}
\RequirePackage{bm} 
\RequirePackage{url}
\usepackage{multirow}

\newcommand{\cG}{\mathcal{G}}

\newcommand{\cM}{\mathcal{M}}

\newcommand{\cO}{\mathcal{O}}

\newcommand{\cZ}{\mathcal{Z}}

\usepackage{pifont}
\usepackage{bbm}
\usepackage{graphicx}
\usepackage{enumitem}
\usepackage{amssymb}
\usepackage{wrapfig}

\definecolor{checkmark}{HTML}{305AFF}
\definecolor{xmark}{HTML}{E62020}

\title{Memory-Bounded Continuation of Greedy Sampling for Continual Anomaly Detection}

\addauthor{Yoon Gyo Jung${}^{*}$}{jung.yoo@northeastern.edu}{1}
\addauthor{Jaewoo Park${}^{*\dagger}$}{park.jaewoo@aivexvision.ai}{2}
\addauthor{Kuan-Chuan Peng}{kpeng@merl.com}{3}
\addauthor{Seongdeok Bang}{bang.seongdeok@aivexvision.ai}{2}
\addauthor{Octavia Camps${}^{\ddagger}$}{o.camps@northeastern.edu}{1}

\addinstitution{
 Electrical and Computer Engineering Department\\
 Northeastern University\\
 Boston, MA, USA
}
\addinstitution{
 AIVEX Inc.\\
 Seoul, South Korea
}
\addinstitution{
 Mitsubishi Electric Research Laboratories (MERL)\\
 Cambridge, MA, USA
}

\runninghead{Jung et al}{Memory-Bounded Continuation of Greedy Sampling for CAD}

\def\eg{\emph{e.g}\bmvaOneDot}

\theoremstyle{plain}
\newtheorem{theorem}{Theorem}[section]
\newtheorem{proposition}[theorem]{Proposition}

\theoremstyle{definition}

\theoremstyle{remark}

\begin{document}

\maketitle

\begin{abstract}
Greedy sampling produces a compact yet representative summary of normal data, which is essential for reliable anomaly detection that relies on measuring distance from normality. For continual anomaly detection where tasks arrive sequentially, extending greedy sampling is straightforward with unbounded memory through coreset accumulation. However, practical deployment requires fixed memory where the coreset size remains constant regardless of task count. We observe that \emph{continued greedy sampling}, which iteratively applies greedy selection over previously greedy-sampled sets, effectively preserves representativeness under strict memory limits. Despite discarding data at each step to satisfy the memory constraint, coreset quality degrades gracefully rather than catastrophically, enabling reliable anomaly detection across the tasks. We provide theoretical justification by showing that resulting \emph{greedy-continued coreset} approximates the oracle coreset within a bounded gap. We instantiate this principle in \ours, which constructs a greedy-continued coreset through greedy expansion on new task features followed by greedy consolidation to enforce the memory budget. Unlike neural methods susceptible to catastrophic forgetting or naive coreset accumulation requiring unbounded memory, \ours maintains fixed memory with theoretical guarantees. Empirically, \ours achieves state-of-the-art performance across 11 task schedules on MVTecAD and VisA, and extends effectively to online continual AD settings where prior methods degrade significantly. Code: \url{https://github.com/jungyg/ContCore}.
\end{abstract}

\section{Introduction}
\label{sec:intro}
Solving unsupervised anomaly detection is vital in various fields such as manufacturing, cybersecurity \cite{cyberad}, and medical applications \cite{medad, medianomaly} due to its productivity and cost reduction. However, discovering the decision boundary between normal and unbounded anomalies makes it challenging. Early models \cite{draem, padim, patchcore, effad, realnet, simplenet, diffad} achieve high performance by training a separate model for each class. However, they require costly annotations, long training time, and large model sizes. Multi-class anomaly detection (AD) methods \cite{uniad, hvq} address these issues by training a single model on multiple classes. This provides faster training and eliminates the need for class labeling. However, they fail in dynamic environments, where new samples and tasks are added, as they suffer from catastrophic forgetting \cite{ewc}.

\begin{figure*}[t]
    \centering
    \begin{minipage}{1\textwidth}
        \subfigure[\label{fig:fm}]{%
            \includegraphics[width=0.32\linewidth]{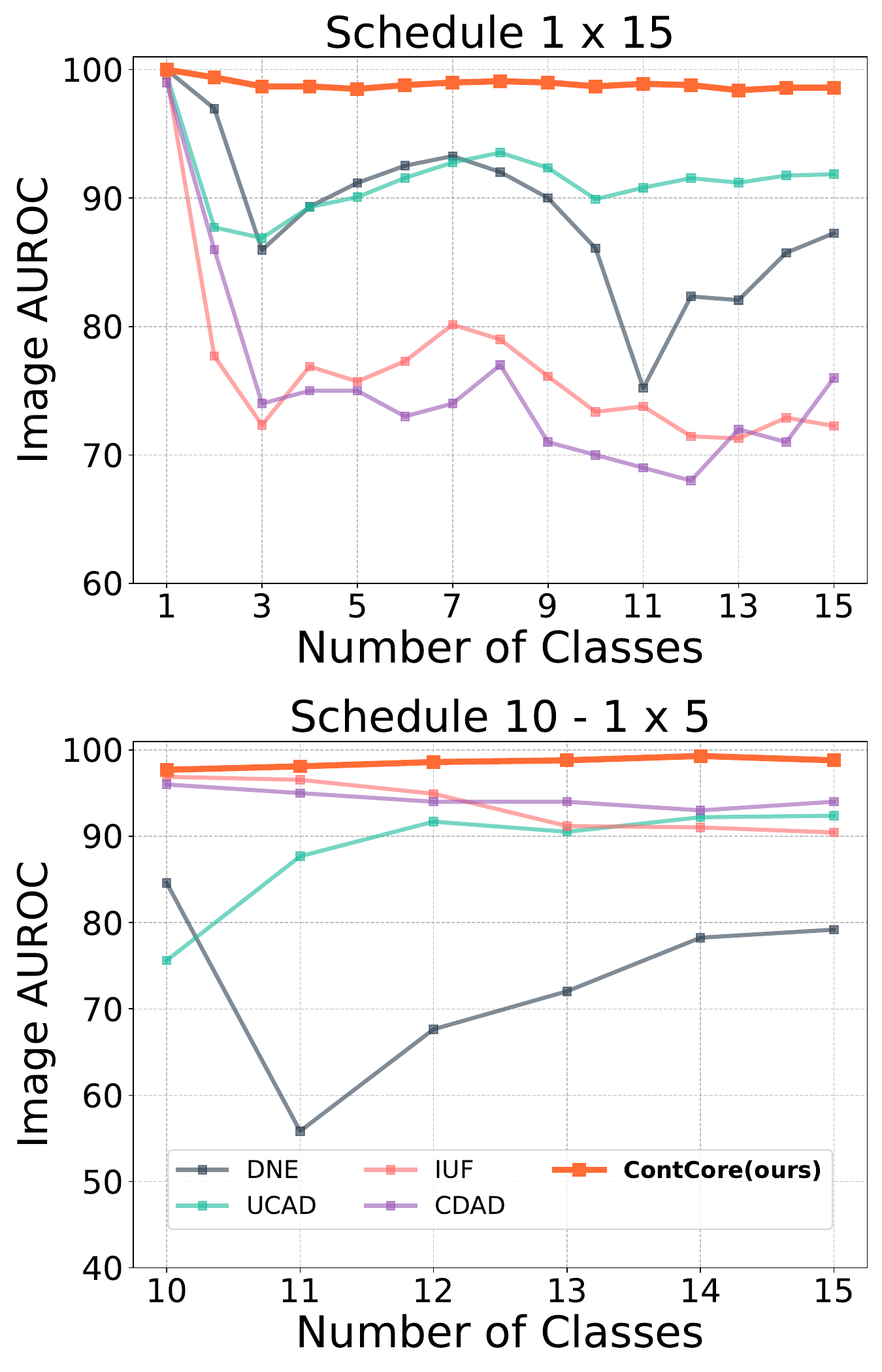}%
        }
        \hfill
        \subfigure[\label{fig:efficiency}]{%
            \includegraphics[width=0.67\linewidth]{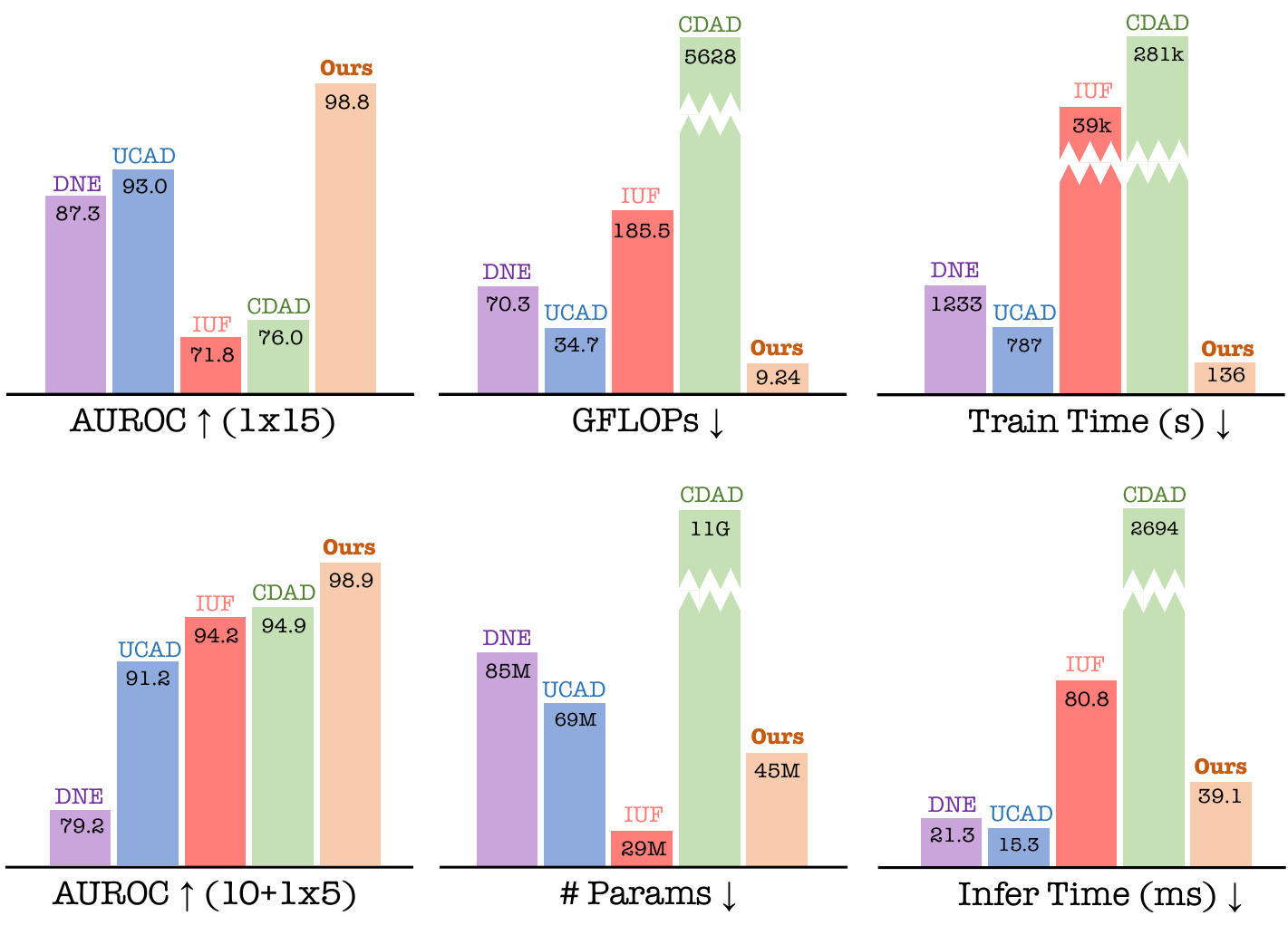}%
        }
    \end{minipage}
    \caption{
        (a) Comparison of task average image-AUROC for CAD baselines and \ours on $1 \times 15$ and $10-1\times5$ task schedules on MVTecAD, which demonstrates the effect of catastrophic forgetting in continual learning. (b) Efficiency comparison of CAD methods.
    }
    \label{fig:fm_and_efficiency}
\end{figure*}

The key challenge in Continual Anomaly Detection (CAD) is preventing catastrophic forgetting and holding the legacy information while detecting newly introduced anomalies under the unsupervised setup (\ie, without class labels and given only normal training images). 
A naive approach is applying either regularization-based \cite{ewc, si, mas} or replay buffer-based \cite{experience_replay, grad_memory, brain_memory} continual learning methods to network-based multi-class AD \cite{uniad, hvq, dinomaly}. While regularization-based methods \cite{ewc, si, mas} preserve past information by regularizing related parameters from past tasks, they struggle to balance stability and plasticity of the parameters. Moreover, they are highly sensitive to the hyperparameters that controls the trade-off between old and new tasks. Replay-based methods \cite{experience_replay, grad_memory, brain_memory} store data from previous tasks in a replay-buffer and replay along with the novel tasks. However, this approach faces scalability issues, as memory usage grows proportionally with the number of tasks. In addition, its performance heavily depends on the quality of the replay samples.

Coreset-based AD \cite{padim, patchcore} builds a coreset memory composed of patch features extracted from a given task and directly uses it for inference. This coreset memory differs from the memory of replay buffer-based methods as it is not used for re-training and the objective is coverage of normal features. These methods can solve continual tasks by adding newly extracted features into the coreset memory without suffering from catastrophic forgetting. However, like replay-based methods, they also struggle with memory overflow as more tasks are added. 

To mitigate these issues, dedicated solutions for CAD \cite{dne, ucad, iuf, cdad} have been proposed. Memory-based methods such as DNE \cite{dne} and UCAD \cite{ucad} handle multiple steps of simple tasks well, but struggle on complex ones (\eg, 10-class task in Fig.~\ref{fig:fm} bottom). Conversely, parameter regularization-based methods such as CDAD~\cite{cdad} and IUF~\cite{iuf} use iterative SVD to project the gradients of new tasks into the null space of parameters to minimize changes to previous task information, but still suffer from forgetting with a large number of tasks (Fig.~\ref{fig:fm} top). 

A natural question arises: can we design a bounded coreset that preserves the representational quality of an unbounded one? Greedy sampling selects maximally separated points, producing a compact yet representative summary of normal data~\cite{patchcore}. Since anomaly detection relies on measuring distance from normality, a representative coreset is essential for reliable detection. Extending this property to the continual setting under fixed memory is non-trivial. We observe that \textit{continued greedy sampling}, which iteratively applies greedy selection over previously greedy-sampled sets, effectively preserves representativeness under strict memory limits. We call the resulting set a \emph{greedy-continued coreset} and provide theoretical justification by showing it approximates the oracle coreset within a bounded gap.

We implement this principle through \ours, which constructs the greedy-continued coreset via two steps: greedy expansion samples features from the new task that are maximally distant from the current coreset, and greedy consolidation enforces the memory constraint while preserving representativeness.

Our contributions are:
\begin{itemize}
\item We observe that continued greedy sampling effectively preserves representativeness under strict memory limits, enabling continual AD without unbounded memory growth. We provide theoretical justification by showing the resulting greedy-continued coreset approximates the oracle coreset within a bounded gap.

\item We instantiate this principle in \ours, which constructs a \emph{greedy-continued coreset} that maintains representativeness of normal data across task transitions. \ours achieves state-of-the-art performance across 11 task schedules on MVTecAD and VisA.

\item We validate the approach in online continual AD settings, demonstrating robustness under stricter conditions where prior methods degrade significantly.
\end{itemize}

\section{Related works}
\label{sec:related_works}

\noindent\textbf{Anomaly detection (AD)} Single-class AD methods \cite{draem, padim, rd4ad, patchcore, simplenet, effad, diffad} and multi-class \cite{uniad, hvq, realnet, dinomaly, diad} mainly use reconstruction \cite{draem, rd4ad, effad, diffad, uniad, hvq, realnet, diad, dinomaly} or memory-based methods \cite{padim, patchcore}. Reconstruction-based methods aim to reconstruct only normal images and measure anomaly scores based on the difference between the input and reconstruction. Moreover, multi-class methods aim to address the ``identity shortcut'' issue, which is caused by the decoder processing more various features, even similar to other classes' anomalies, as more classes are shown. However, these recent works lack any explicit mechanisms designed for a continual learning setup, suffering from catastrophic forgetting.

\paragraph{Continual learning}
The main challenge in continual learning is avoiding catastrophic forgetting, where a model's performance on past tasks degrades upon learning new ones. Two method types have been proposed to address this: regularization-based and replay-based methods. Regularization-based methods \cite{ewc, si, mas} use a loss function to prevent important parameters from old tasks from diluting, but often struggle to balance stability and plasticity. Replay-based methods \cite{experience_replay}, inspired by biological mechanisms \cite{brain_memory}, store a subset of past data in a memory buffer to interleave with new data during training. Though effective, they scale poorly as memory grows with task count, and their performance heavily depends on replay-sample quality \cite{grad_memory, latent_replay}.

\paragraph{Continual anomaly detection (CAD)}
DNE \cite{dne} models each task features with the mean vector and covariance matrix of the assumed Gaussian distributions for fast, lightweight training and inference, but struggles with complex, multi-class tasks and lacks pixel-level AD ability. UCAD \cite{ucad} uses a memory bank with key-prompt modules, continual prompt modulation, and Segment Anything Model (SAM)-based \cite{sam} pseudo-label contrastive learning. Its performance heavily relies on SAM’s segmentation quality, which is used for training the prompts, where wrong segmentations would affect the representations of the prompt. Both DNE and UCAD struggle with multi-class tasks.

IUF \cite{iuf} uses object-aware self-attention and semantic compression loss to preserve key semantics across tasks, but requires class labels, limiting it to supervised settings, unlike most multi-class AD methods \cite{uniad, hvq}, which assume that the class information is not given. CDAD \cite{cdad} employs a diffusion model with the anomaly-masked network, which enhances the conditioning mechanism, and memory-efficient gradient projection via SVD. However, it is computationally heavy with slow training and inference, as shown in Fig.~\ref{fig:efficiency}. Both IUF and CDAD suffer from catastrophic forgetting when facing long-term continual task schedules.

\begin{figure*}[t]
\centering
\includegraphics[width=1\linewidth]{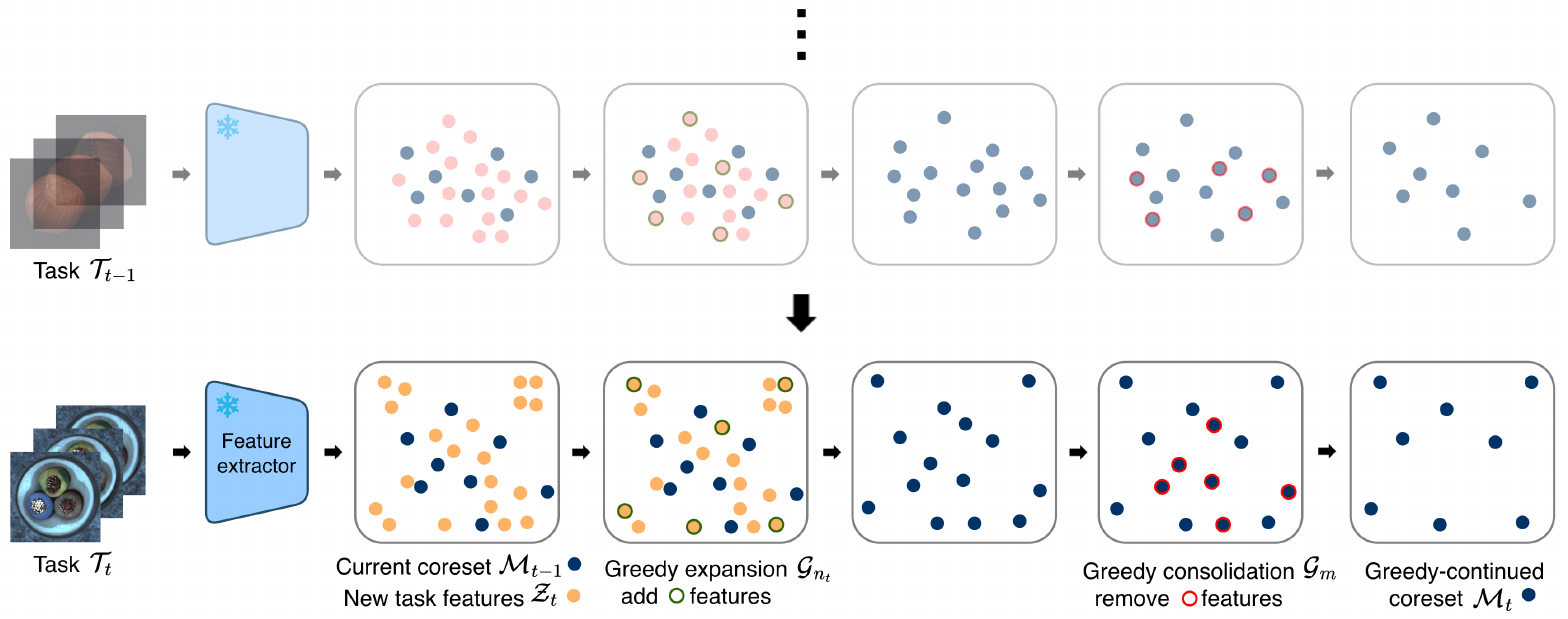}
\caption{
Memory-bounded continuation of greedy sampling via \ours. Greedy expansion selects maximally distant features from new task features relative to the current coreset (green circles). Greedy consolidation enforces the memory constraint by selecting maximally separated features from the combined set (red circles), producing a greedy-continued coreset that preserves coverage across tasks.
}
\label{fig:method}
\end{figure*}

\section{Preliminaries}
\paragraph{Greedy sampling} is a subset-selection operation $\mathcal{G}$ that, given a set $\mathcal{Z}$ of latent embedding vectors as an input, returns a subset with maximized distances as follows: 
\begin{equation}
\mathcal{G}_n(\mathcal{Z})
= \{
z_{i}^* \in \mathcal{Z}:
z_{i}^*
= 
\underset{z \in \mathcal{Z}}{\arg \max}
d(z, \mathcal{Z}_{i-1}^*)
\}
\end{equation}
for $2 \leq i \leq n $, where $n$ is the sampling size and $\mathcal{Z}_{i-1}^*=\{z_1^*, \dots, z_{i-1}^*\}$, where the set distance between a vector $z$ and the set $\mathcal{Z}_{i-1}^*$, $d$, is defined as the minimum Euclidean distance to the vectors in the set.
The initial point $z_1^*$ is chosen randomly or as one whose distance summation to all points is maximal. To simplify notations, we set $\mathcal{G}_n(\mathcal{Z}) = \mathcal{Z}$ if the data size is not greater than the sample size $n$ (\ie, $|\mathcal{Z}| \leq n$).

\paragraph{Continual anomaly detection} Unsupervised CAD consists of a sequence $(\mathcal{T}_1, \; \dots, \; \mathcal{T}_T)$ of $T$ tasks, where \textit{the number of tasks $T$ is unknown}. Each task $T_t$ is paired with its training data $\mathcal{D}_t$, which has only normal samples without task or class labels. At the $t$-th task, the model trainer has access to the model from the previous $(t{-}1)$-th task, but has no direct access to the data from the previous tasks.

\section{Method}
We implement \ours, a memory-based anomaly detector that constructs a greedy-continued coreset via two steps: greedy expansion and greedy consolidation. Greedy expansion samples features from the new task that are maximally distant from the current coreset, and greedy consolidation enforces the memory constraint while preserving representativeness. The greedy-continued coreset is the key to continual AD: it maintains a representative summary of all observed normal data under fixed memory. The fact that the base set $B$ is arbitrary enables greedy sampling to generalize across tasks by setting $B$ as the current coreset.

\subsection{Greedy sampling on a base set}
The original greedy sampling $\mathcal{G}_n$ starts from an empty set $\emptyset$, and inductively collects the points that have maximal distance to the set collected at hand. The notion of greedy sampling can be extended by a base set $B$, where the greedy sampling starts from a non-empty base set $B$:

\fontsize{9.9pt}{10pt}
\begin{equation}
\mathcal{G}_n(\mathcal{Z}; B) 
= \{z_{i}^*\in\mathcal{Z}:
z_{i}^*
= 
\underset{z \in \mathcal{Z} }{\arg \max}
d(z, \mathcal{Z}_{i-1}^* \cup B)
\}
\end{equation}
\normalsize
for $2 \leq i \leq n$. The initial point $z_1^*$ is given as  $z_1^* = \arg \max_{z \in \mathcal{Z}}\; d(z, \; B)$.

The extended greedy sampling differs from the original in that a point being collected is compared to both the sampled set $\{z_1^*, \; \dots, \; z_{i-1}^*\}$ and the base set $B$. Hence, under this definition, the original greedy sampling is sampling with the empty base set, \ie, $\mathcal{G}_n(\mathcal{Z})$ is equal to $\mathcal{G}_n(\mathcal{Z}; \; \emptyset)$. Note that the base set $B$ can be any set of vectors of the same dimension. Depending on the choice of the base set, however, the final sampled output can differ.

\subsection{Greedy-Continued Coreset for CAD}
A naive application of greedy sampling in the continual learning scenario would result in a sampled set with size proportional to the number of tasks $T$. In this case, the memory and computation cost grow unboundedly, making it impractical for many tasks. Therefore, we propose \ours, which constructs a greedy-continued coreset via memory-bounded continuation of greedy sampling as follows:

Given a set $\mathcal{Z}_t$ of patch embedding features from training set $\mathcal{D}_t$ of the $t$-th task, \ours at $t$-th task is performed by
\begin{equation}
\label{eq:meta_greedy_sampling}
\mathcal{M}_t = \mathcal{G}_m
(
\mathcal{G}_{n_t}
(\mathcal{Z}_t; \; \mathcal{M}_{t-1}) \cup \mathcal{M}_{t-1}
),
\end{equation}
where the memory $\mathcal{M}_t$ is inductively computed starting with $\mathcal{M}_{0} = \emptyset$, subscript $n_t$ on greedy expansion is the sampling size at $t$-th task and subscript $m$ on greedy consolidation is the maximum coreset size. We denote the overall \ours operation in Eq.~\eqref{eq:meta_greedy_sampling} by $ \mathcal{M}_t = \mathcal{G}_{m, n_t}( \mathcal{Z}_t; \; \mathcal{M}_{t-1})$.

Thus, \ours in Eq.~\eqref{eq:meta_greedy_sampling} consists of greedy expansion $\mathcal{G}_{n_t}(\cdot ; \; \mathcal{M}_{t-1})$ and greedy consolidation $\mathcal{G}_m(\cdot)$. Fig.~\ref{fig:method} shows that greedy expansion selects features from the new task that are maximally distant from the current coreset, while greedy consolidation enforces the memory constraint by selecting maximally separated features from the combined set. This two-step process constructs a greedy-continued coreset at each task transition. The constrained optimization criterion in Eq.~\eqref{eq:meta} in Sec.~\ref{sec:meta_learning} further clarifies the structure-preserving nature of \ours.

\paragraph{Why greedy-continued coresets enable continual AD} Greedy sampling selects maximally separated points, producing a compact yet representative summary of normal data~\cite{patchcore}. This representativeness is essential for effective AD, since anomaly detection relies on measuring distance from normality. The central question for continual AD is whether representativeness can be maintained as tasks accumulate under bounded memory. We observe that continued greedy sampling effectively preserves representativeness across task transitions. Theorem~\ref{theorem} provides theoretical justification, showing the representational gap to the oracle coreset remains bounded.

\subsection{Implementation in practice}
\noindent\textbf{Improving efficiency by approximation.}
\label{sec:approx_meta_greedy}
The greedy consolidation step inevitably increases computation time. To mitigate this, we approximate greedy sampling using nearest neighbors. The key insight is that after greedy expansion, the combined features are already mostly representative. Thus, instead of strict costly greedy sampling, faster nearest neighbor can sufficiently measure distinctiveness.

Upon this intuition, during greedy consolidation, a fraction $q$ of the samples are selected by nearest neighbors instead of full greedy sampling. Specifically, we compute the nearest neighbor distance $d_i$ for each embedding $z_i$ in the set $\widehat{\mathcal{M}}_t$ from greedy expansion and the current coreset;
\begin{equation}
d_i = d(z_i, \; \widehat{\mathcal{M}}_t \setminus \{ z_i \})
\end{equation}
with $ \widehat{\mathcal{M}}_t = \mathcal{G}_{n_t}
(\mathcal{Z}_t; \; \mathcal{M}_{t-1}) \cup \mathcal{M}_{t-1}$. Then, by sorting $d_i$ in descending order, we select the first fraction $q$ of the embeddings with highest nearest distances.

\subsection{Inference}
To detect anomalies in the inference stage of the $t$-th task, we compute the anomaly score based on the nearest distance as shown in PatchCore \cite{patchcore}. In particular, given a test sample, we extract its patch embeddings $z^{(h, w)}$ for the $(h, w)$-th patch and compute the anomaly score of $(h, w)$-th patch by
\begin{equation}
a^{(h, \; w)} = d(z^{(h, \; w)}, \; \mathcal{M}_t).
\end{equation}
If the anomaly score for a particular patch is higher than a threshold, then that region is considered to contain an anomaly. If the maximum of anomaly scores across all patches is greater than a threshold, then the given test sample is considered to be an anomaly.

\subsection{Theoretical analysis}
\label{sec:theorem}
We provide theoretical justification for our observation. The representational gap between the greedy-continued coreset and the oracle coreset (greedy sampling applied to all task data) remains bounded.
\begin{theorem}
Suppose $H(\mathcal{O}, \; \cup_{t=1}^T \mathcal{Z}_t) \leq \epsilon_o$, and assume $H(S_t, \; \mathcal{Z}_t) \leq \epsilon_t$, and $H(\mathcal{M}_t, \; S_t \cup \mathcal{M}_{t-1}) \leq \widehat{\epsilon}_t$
for all $t=1,\dots,T$ with $S_1 = \mathcal{M}_1$, $\mathcal{M}_0 = \emptyset$, and $\widehat{\epsilon}_0 = 0$,
where $H$ is the Hausdorff distance, $\mathcal{O} = \mathcal{G}_{n_0}(\cup_{t=1}^T \mathcal{Z}_t)$, $S_t = \mathcal{G}_{n_t}(\mathcal{Z}_t; \; \mathcal{M}_{t-1})$, and $\mathcal{M}_t = \mathcal{G}_{m_t}(S_t \cup \mathcal{M}_{t-1})$.
Then, we have
\begin{equation}
\label{eq:inequality}
H(\mathcal{O}, \; \mathcal{M}_T)
\leq
\epsilon_o +
\max_{1 \leq k \leq T} \left( \epsilon_k + \sum_{j=k}^T \widehat{\epsilon}_j \right)
\end{equation}
for sufficiently large $n_o \leq | \cup_{t=1}^T \mathcal{Z}_t |$, $n_t \leq |\mathcal{Z}_t|$, and $m_t \leq |S_t \cup \mathcal{M}_{t-1}|$.
\label{theorem}
\end{theorem}

The theorem establishes that the greedy-continued coreset $\mathcal{M}_T$ incurs only bounded error relative to the oracle coreset $\mathcal{O}$. The error accumulates per-task but does not compound catastrophically, supporting our observation that continued greedy sampling preserves representativeness under memory constraints. With sufficient per-task sampling size $n_t$, the greedy-continued coreset maintains quality. We prove Theorem~\ref{theorem} in Sec.~\ref{sec:theory} in the supplement.

\subsection{\ours as Meta Learning}
\label{sec:meta_learning}
Given a set $Z$, \textit{greedy sampling} samples instances by
\begin{equation}
z_{(n+1)}^* = \underset{z \in Z}{\arg \max} \, d(z, Z_{n}^*)
\end{equation}
where $d$ is the distance between a vector $z$ and the set of vectors $Z_n^* = \{ z^*_{(k)} \}_{k=1}^n$ constructed recursively. As the set distance between a vector $z$ and the set $\mathcal{Z}_{n}^*$, $d$ is defined as the minimum Euclidean distance to the vectors in the set.
The initial point $z_1$ is chosen randomly or as one whose distance summation to all points is maximal.

Given a set $Z$ and an arbitrary base set $B$, \textit{greedy-continued coreset sampling} samples instances by
\begin{align*}
& z_{(n+1)}^* = \underset{z \in \widehat{Z}_m}{\arg \max} \;  d(z, Z_{n}^*) \\
\text{subject to } & \widehat{Z}_m = \{ \widehat{z}_{(j)} \}_{j=1}^m
\end{align*}
where $\widehat{Z}_m {=} \{ \widehat{z}_{(j)} \}_{j=1}^m$ is constructed recursively by
\begin{equation}
\label{eq:meta}
\widehat{z}_{(j+1)} = \underset{z \in Z}{\arg \max} \;  d(z, \widehat{Z}_{j} \cup B).
\end{equation}
with an arbitrary base set $B$. This constrained optimization criterion clearly indicates the meta-learning nature of \ours as it applies greedy sampling on samples that are sampled by the greedy sampling.

\section{Experiments}
\label{sec:exp}
We conduct experiments to show the robustness of \ours. First, we explain the setup of our experiments: dataset descriptions, continual task schedules, evaluation metrics, and hardware resources specs. The following section compares \ours with the SOTA baselines using 11 different task schedulings. Finally, we conduct extensive ablation studies by analyzing the components of \ours, present empirical values which support our theoretical bound, compare with PatchCore under the oracle setup, test \ours on the online setup comparing with the SOTA baselines, and present the results of \ours on Real-IAD~\cite{realiad} in the supplement.

\begin{table*}[t]
\centering
\resizebox{1\linewidth}{!}{
\begin{tabular}{r@{\hspace{2mm}}c@{\hspace{1mm}}c@{\hspace{1mm}}c@{\hspace{1mm}}c@{\hspace{1mm}}c@{\hspace{1mm}}c@{\hspace{1mm}}c@{\hspace{1mm}}c@{\hspace{1mm}}c@{\hspace{1mm}}c}
\toprule
\multicolumn{1}{c}{\multirow{2.5}{*}{Method}} 
& \multicolumn{2}{c}{\textbf{1 × 15 with 15 Steps}} 
& \multicolumn{2}{c}{\textbf{3 × 5 with 5 Steps}} 
& \multicolumn{2}{c}{\textbf{10 -- 1 × 5 with 6 Steps}} 
& \multicolumn{2}{c}{\textbf{10 -- 5 with 2 Steps}} 
& \multicolumn{2}{c}{\textbf{14 -- 1 with 2 Steps}} 
\\ 
\cmidrule(r){2-3} \cmidrule(r){4-5} \cmidrule(r){6-7} \cmidrule(r){8-9} \cmidrule(r){10-11}
& \textbf{AUROC ($\uparrow$)} & \textbf{FM ($\downarrow$)} & \textbf{AUROC ($\uparrow$)} & \textbf{FM ($\downarrow$)} & \textbf{AUROC ($\uparrow$)} & \textbf{FM ($\downarrow$)} & \textbf{AUROC ($\uparrow$)}& \textbf{FM ($\downarrow$)} & \textbf{AUROC ($\uparrow$)}& \textbf{FM ($\downarrow$)}\\
\midrule
UniAD$\;_\text{NeurIPS'22}$ \cite{uniad} 
& 77.4 / - & 22.9 / - 
& 81.3 / 88.7 & 7.4 / 10.6  
& 76.6 / 82.3 & 21.1 / 17.3 
& 86.7 / 91.5 & 14.9 / 10.6 
& 85.7 / 89.6 & 18.3 / 13.3 
\\
UniAD  + EWC$\;_\text{PNAS'17}$ \cite{ewc}  
& - / - & - / - 
& 79.6 / 89.0 & 9.5 / 10.1  
& 89.6 / 93.8 & 5.4 / 3.6 
& 90.5 / 93.6 & 7.3 / 4.2 
& 92.8 / 95.4  & 4.1 / 1.9
\\
UniAD  + SI$\;_\text{PMLR'17}$ \cite{si} 
& - / - & - / - 
& 81.9 / 88.5 & 7.0 / 10.8  
& 77.2 / 81.6 & 20.2 / 18.2  
& 84.1 / 88.3 & 20.2 / 17.0 
& 85.7 / 89.5  & 18.4 / 13.4 
\\
UniAD  + MAS$\;_\text{ECCV'18}$ \cite{mas} 
& - / - & - / - 
& 81.5 / 89.0 & 7.2 / 10.2  
& 77.9 / 82.0& 19.5 / 17.7 
& 86.8 / 91.0 & 14.9 / 11.6 
& 85.8 / 89.6  & 18.1 / 13.3 
\\
UniAD  + LVT$\;_\text{CVPR'22}$ \cite{lvt} 
& - / - & - / - 
& 80.4 / 88.6 & 8.6 / 10.6  
& 78.2 / 88.3 & 19.1 / 16.1 
& 87.1 / 90.6 & 14.1 / 12.3 
& 80.4 / 86.0  & 29.1 / 20.6 
\\
\midrule 
DNE$\;_\text{ACMMM'22}$ \cite{dne}  
& 87.3 / - & 5.3 / - 
& 84.3 / - & 1.75 / - 
& 79.2 / - & \textbf{0.4} / - 
& 82.8 / - & \ul{2.1} / - 
& 79.2 / - & {0.1} / - 
\\ 

UCAD$\;_\text{AAAI'24}$ \cite{ucad}
& \ul{93.0} / 83.3 & 1.0 / \textbf{-0.6} 
& 84.8 / 90.1 & 10.3 / 9.2 
& {91.2} / {94.0} & {6.3} / {3.0} 
& 88.7 / 93.1 & 5.2 / 3.7 
& {93.8} / 95.7  & 1.8 / 0.3
\\ 

IUF$\;_\text{ECCV'24}$ \cite{iuf}
& 71.8 / 79.2 & 5.1 / 6.6 
& {84.2} / {91.1} & {10.0} / 8.4  
& {94.2} / {95.1}& {3.2} / {1.0} 
& 92.2 / 94.4 & 9.3 / 4.3
& 96.0 / 96.3  & 1.0 / 0.6 
\\ 

CDAD$\;_\text{CVPR'25}$ \cite{cdad} 
& 76.0 / \ul{89.8} & 10.2 / 5.7 
& \ul{89.1} / \ul{92.2} & \ul{3.7} / \ul{4.7} 
& \ul{94.9} / \ul{95.7}& {1.0} / \ul{0.7} 
& \ul{94.2} / \ul{95.3} & \ul{2.1} / \ul{2.4} 
& \ul{96.4} / \ul{96.6}  & \textbf{-0.6} / \textbf{-0.04}  
\\ 

\midrule
\rowcolor{red!25}
\textbf{\ourso} 
& \textbf{98.8} / \textbf{97.4} & \textbf{0.1} / \ul{0.2}  
& \textbf{98.2} / \textbf{96.9} & \textbf{0.6} / \textbf{0.5} 
& \textbf{98.9} / \textbf{96.4}& \ul{0.5} / \textbf{0.5} 
& \textbf{97.7} / \textbf{96.6} & \textbf{0.5} / \textbf{0.5} 
& \textbf{98.3} / \textbf{97.3}  & \ul{0.0} / \ul{0.1}  
\\ 
\bottomrule
\end{tabular}
}
\caption{Task average AUROC ($\uparrow$) and forgetting measurement (FM) ($\downarrow$) on MVTecAD dataset with 5 different task schedules. Each cell shows the image-level/pixel-level AUROC. The best and second best results are marked in \textbf{bold} and \ul{underlined}, respectively.}
\label{table:mvtec_auroc}
\end{table*}

\begin{table*}[t]
\centering
\resizebox{1\linewidth}{!}{
\begin{tabular}{rcccccccc}
\toprule
\multicolumn{1}{c}{\multirow{2.5}{*}{Method}} 
& \multicolumn{2}{c}{\textbf{1 × 12 with 12 Steps}}  
& \multicolumn{2}{c}{\textbf{8 -- 1 × 4 with 5 Stepss}} 
& \multicolumn{2}{c}{\textbf{8 -- 4 with 2 Steps}} 
& \multicolumn{2}{c}{\textbf{11 -- 1 with 2 Steps}} 
\\ 
\cmidrule(r){2-3} \cmidrule(r){4-5} \cmidrule(r){6-7} \cmidrule(r){8-9} 
& \textbf{AUROC ($\uparrow$)} & \textbf{FM ($\downarrow$)} & \textbf{AUROC ($\uparrow$)} & \textbf{FM ($\downarrow$)} & \textbf{AUROC ($\uparrow$)} & \textbf{FM ($\downarrow$)} & \textbf{AUROC ($\uparrow$)}& \textbf{FM ($\downarrow$)}\\
\midrule

UniAD$\;_\text{NeurIPS'22}$ \cite{uniad}   & 63.9 / - & 29.7 / - & 72.2 / 90.8 & 16.6 / 9.2 & 78.1 / 94.0 & 14.7 / 8.4 & 75.0 / 792.1 & 22.4 / 11.4 \\ 
UniAD + EWC$\;_\text{PNAS'17}$ \cite{ewc}  & - / - & - / - & 72.3 / 92.3 & 16.5 / 7.3 & 80.5 / 95.4 & 10.0 / 5.3 & 78.7 / 95.4 & 14.9 / 4.8 \\ 
UniAD + SI$\;_\text{PMLR'17}$ \cite{si}  & - / - & - / - & 69.8 / 88.5 & 19.8 / 12.0 & 69.8 / 88.5 & 9.2 / 8.3 & 78.1 / 92.0 & 16.9 / 11.5 \\ 
UniAD + MAS$\;_\text{ECCV'18}$ \cite{mas}  & - / - & - / - & 72.1 / 90.6 & 16.7 / 9.4 & 72.1 / 90.6 & 14.1 / 8.4 & 75.4 / 91.8 & 21.5 / 11.9 \\ 
UniAD + LVT$\;_\text{CVPR'22}$ \cite{lvt}  & - / - & - / - & 70.8 / 94.4 & 18.3 / 8.4 & 70.8 / 91.4 & 13.4 / 8.1 & 77.5 / 92.3 & 17.3 / 10.9 \\ 

\midrule

DNE (ViT-B/16)$\;_\text{ACMMM'22}$ \cite{dne} 
& 55.2 / - & 11.4 / - 
& 58.6 / - & 10.2 / - 
& 64.1 / - & 6.1 / - 
& 68.8 / - & 1.9 / - 
\\ 
UCAD (ViT+SAM)$\;_\text{AAAI'24}$ \cite{ucad}  
& 78.1 / 76.9 & \textbf{0.0} / \textbf{0.6}   
& 78.8 / 93.5 & {10.4} / {5.4} 
& 79.9 / 94.2  & 8.9 / 4.8 
& 85.9 / 94.5 & 2.7 / \ul{0.6}
\\
IUF (EfficientNet)$\;_\text{ECCV'24}$ \cite{iuf}   
& 64.7 / 83.6 & 3.9 / \ul{5.0}   
& {79.8} / {95.0} & {9.8} / 6.8 
& {80.1} / {95.4}  & 9.8 / 6.8 
& {87.3} / \textbf{97.6}  & {2.4} / 1.8
\\
CDAD (SDv1.5)$\;_\text{CVPR'25}$ \cite{cdad}
& \ul{78.5} / \textbf{97.3} & 11.1 / 6.0 
& \ul{83.4} / \ul{95.8}  & \ul{3.8} / \ul{1.5} 
& \ul{85.3} / \textbf{97.1}   & \ul{1.4} / \textbf{0.3}   
& \ul{88.3} / {97.2} & \textbf{-1.3} / \textbf{-0.1}  
\\ 
\midrule
\rowcolor{red!25}
\textbf{\ourso (WRN50)} 
& \textbf{93.7} / \ul{97.2}& \ul{0.7} / \textbf{0.6} 
& \textbf{96.3} / \textbf{97.2} & \textbf{0.7} / \textbf{0.5} 
& \textbf{93.8} / \ul{96.5} & \textbf{1.2} / {0.8} 
& \textbf{96.1} / \ul{97.7}  & \ul{0.4} / \textbf{-0.0} 
\\ 
\bottomrule
\end{tabular}
}
\caption{Task average AUROC ($\uparrow$) and forgetting measurement (FM) ($\downarrow$) on VisA dataset with 4 different task schedules. Each cell shows the image-level/pixel-level AUROC. The best and second best results are marked in \textbf{bold} and \ul{underlined}, respectively.}
\label{table:visa_auroc}
\end{table*}

\subsection{Experimental setup} 
\noindent\textbf{Datasets and task scheduling} We conduct most of the experiments with the two widely used AD datasets: MVTecAD \cite{mvtec} and VisA \cite{visa}. For the continual task setup, we follow schedules introduced in IUF and UCAD. We denote a sequence of tasks with -- and repeated task with $\times$ where each number represent the number of classes for each task. For example, 10 -- 1 $\times$ 5 starts with a task of 10 classes followed by 5 tasks with 1 class each. IUF schedules MVTecAD with 14 -- 1 (2 steps), 10 -- 5 (2 steps), 3$\times$5 (5 steps), and 10 -- 1$\times$5 (6 steps) and VisA with 11 -- 1 (2 steps), 8 -- 4 (2 steps), 8 -- 1$\times$4 (5 steps) continual tasks, and UCAD arranges MVTecAD and VisA into 1$\times$15 (15 steps) and 1$\times$12 (12 steps). Moreover, CDAD proposes cross-dataset setup which trains the whole MVTecAD dataset, followed by the whole VisA and vice versa. The order of tasks are sorted by class names in the ascending alphabetical order.

These schedulings are designed to test the model's adaptability in various ways. 1$\times$15 and 1$\times$12 tests single-class incremental adaptability and verifies the resilience against catastrophic forgetting over a large number of tasks. 14 -- 1 and 11 -- 1 tests adding a single new class to a heavily pre-trained, multi-class base model mimicking the scenario where a deployed system receives a minor update. 10 -- 5 and 8 -- 4 simulates adding a large batch of new classes to the base model, simulating a major facility upgrade or bulk category addition in a single time. Meanwhile, 10 -- 1 $\times$ 5 and 8 -- 1 $\times$ 4 evaluates whether the complex base model can continuously adapt to a steady, ongoing stream of individual new tasks over time. Finally, 3 -- 3 $\times$ 4 simulates multi-class incremental setup similar to the 1$\times$15 schedule.

\paragraph{Models, hyperparameters, and hardware setup} We follow the setup of PatchCore \cite{patchcore}, where patch features are extracted from layer 2 and 3 of the WideResNet50~\cite{wideresnet} backbone with sampling rate $p=0.01$, coreset size $m=20,000$ for MVTecAD, $m=40,000$ for VisA. We use a single RTX 3090 GPU for the experiments.

\paragraph{Metrics} We report the task average performance \cite{performance} after the final task $T$ is trained. Specifically, the task average performance $R_T$ is computed by
\begin{equation}
\label{eq:avg_acc}
R_T
=
\frac{1}{|T|} \sum_{t=1}^{T}{R_{T,t}},
\end{equation}
\noindent where $R_{s,t}$ indicates the performance of the task $t$ at step $s$, and $T$ denotes the total number of steps.

We also report the forgetting measurement (FM) \cite{fm} metric, which measures the amount of information the model has lost throughout learning additional tasks. FM is computed by averaging the difference between the best performance during the previous tasks and the last performance after all tasks are trained, as shown below:
\begin{equation}
\label{eq:fm}
\text{FM} 
=
\frac{1}{T-1} \sum_{t=1}^{T-1}{ \underset{s \in \{1, \; \dots, \; T-1\}}{\max} (R_{s, \; t}-R_{T, \; t})}.
\end{equation}
We report the task average performance $R_T$ and FM of image-level Area Under the Receiver Operator Curve (AUROC) and pixel-level AUROC in Tab.~\ref{table:mvtec_auroc} and ~\ref{table:visa_auroc}. In addition, we also report the image-level Average Precision (AP), pixel-level Area Under the Per-Region-Overlap (AUPRO) in Tab.~\ref{table:mvtec_ap}, ~\ref{table:mvtec_pro}, ~\ref{table:visa_ap}, and~\ref{table:visa_pro} in the supplement.

\begin{table}[t]
\centering
\resizebox{0.7\linewidth}{!}{
\begin{tabular}{rcccc}
\toprule
\multicolumn{1}{c}{\multirow{2.5}{*}{Method}}& \multicolumn{2}{c}{\textbf{MVTec$\rightarrow$VisA}} & \multicolumn{2}{c}{\textbf{VisA$\rightarrow$MVTec}}\\
\cmidrule(r){2-3} \cmidrule(r){4-5}
\multicolumn{1}{l}{\multirow{-2}{*}{}} & \textbf{AUROC ($\uparrow$)}    & \textbf{FM ($\downarrow$)}    & \textbf{AUROC ($\uparrow$)}   & \textbf{FM ($\downarrow$)}    \\ \midrule
IUF$\;_\text{ECCV'24}$ \cite{iuf} & 82.4 / 92.5  & {11.1} / {4.9}  & 74.3 / 89.7 & 16.7 / {5.5} \\
CDAD$\;_\text{CVPR'25}$ \cite{cdad} & \ul{90.0} / \ul{94.9}  & \ul{4.7} /  \textbf{1.8}  & \ul{84.2} /  \ul{94.1} & \ul{6.9} /  \ul{3.5}  \\
\midrule
\rowcolor{red!25}
\textbf{\ourso} & \textbf{94.3} / \textbf{96.6}  & \textbf{0.6} /  \ul{0.3}  & \textbf{94.0} /  \textbf{96.7} & \textbf{2.1} /  \textbf{0.7}  \\ 
\bottomrule
\end{tabular}}
\caption{AUROC and FM results on cross-dataset CAD.}
\label{table:cross_dataset}
\end{table}

\subsection{Comparison with the SOTA baselines}
We compare \ours with UniAD, UniAD applied with continual learning methods (EWC, SI, MAS, and LVT), and existing CAD methods (DNE, UCAD, IUF, and CDAD). As shown in Tab.~\ref{table:mvtec_auroc} and~\ref{table:visa_auroc}, UniAD and the continual learning variants significantly show catastrophic forgetting based on the high FM and low AUROC while the CAD baselines show improved performance with lower FM and higher AUROC. However, existing CAD methods exhibit inherent trade-offs across task schedule and task complexity. Memory-based CAD methods (DNE, UCAD) minimize long-range forgetting from long-range task schedules (MVTecAD 1$\times$15, VisA 1$\times$14) with the inherent memories, however, struggle with complex tasks (MVTecAD 10 -- 5, VisA 8 -- 4, and VisA 8 -- 1$\times$4). Regularization-based methods resolve complex tasks well (MVTecAD 10 -- 5, VisA 8 -- 4, and VisA 8 -- 1$\times$4) and struggle from long-range task schedule as they lack explicit mechanism to keep information of long-range tasks. In contrast, \ours maintain stable results across every task schedules showing robust performance for both image and pixel-level metrics as well as FM. In Tab.~\ref{table:cross_dataset}, \ours surpass the previous baseline method CDAD in the cross dataset setup with a remarkable 8.9\% margin in VisA$\rightarrow$MVTecAD image-level AUROC. Finally, Fig.~\ref{fig:efficiency} compares the efficiency in terms of computation, number of parameter, training and inference time where \ours show significant efficiency comparing with other methods.

\begin{table}[t]
\centering
\resizebox{0.7\linewidth}{!}{
\begin{tabular}{r@{\hspace{0.6em}}c@{\hspace{0.6em}}c@{\hspace{0.6em}}c@{\hspace{0.6em}}c@{\hspace{0.6em}}}
\toprule
\multicolumn{1}{c}{\multirow{2}{*}{Method}} & \textbf{\multirow{2}{*}{GFLOPS ($\downarrow$)}}  & \textbf{\multirow{2}{*}{Parameters ($\downarrow$)}}  & \textbf{\multirow{2}{*}{\shortstack{Test time\\ (ms) ($\downarrow$)}}} & \textbf{\multirow{2}{*}{\shortstack{Train time\\(s) ($\downarrow$)}}} \\
\\
\midrule
DNE$\;_\text{ACMMM'22}$ \cite{dne}  & 70.3 & 85.8 M & \ul{21.3} & 1233 \\
UCAD$\;_\text{AAAI'24}$ \cite{ucad}  & \ul{33.8} & 68.9 M & \textbf{15.3} & \ul{786} \\
IUF$\;_\text{ECCV'24}$ \cite{iuf}  & 185.5 & \textbf{28.5 M} & 80.8 & 39192 \\ 
CDAD$\;_\text{CVPR'25}$ \cite{cdad}  & 5628.0 & 11.0 G & 2694.0 & 281434 \\ 
\midrule
\rowcolor{red!25}
\textbf{\ourso} & \textbf{9.24} & \ul{45.3M} & 39.1 & \textbf{754} \\
\bottomrule
\end{tabular}
}
\caption{Comparing the efficiency of \ours and the SOTA baselines in terms of computation, model size, inference speed, and training time.}
\label{table:efficiency}
\end{table}

\begin{table}[t]
\centering
\resizebox{0.7\linewidth}{!}{
\begin{tabular}{rcccc}
\toprule
\multicolumn{1}{c}{\multirow{2.5}{*}{Task Schedule}} & \multicolumn{2}{c}{\textbf{Img / Pixel AUROC}} & \multicolumn{2}{c}{\textbf{Image AP / AUPRO}}\\ 
\cmidrule(r){2-3} \cmidrule(r){4-5} 
& \textbf{Performance ($\uparrow$)} & \textbf{FM ($\downarrow$)} & \textbf{Performance ($\uparrow$)} & \textbf{FM ($\downarrow$)} \\
\midrule
$1\times30$ & 90.2 / 97.5 & 1.7 / 0.6 & 85.1 / 96.6 & 1.7 / 0.6 \\ 
$5\times6$ & 90.2 / 97.7 & 1.4 / 0.3 & 85 / 96.8 & 1.6 / 0.3 \\ 
$10\times3$ & 90.3 / 97.6 & 0.9 / 0.1 & 85 / 96.7 & 1.1 / 0.1 \\ 
10--20 & 89.9 / 97.5 & -0.3 / -0.1 & 84.6 / 96.6 & -0.1 / -0.1 \\ 
20--10 & 90.4 / 97.3 & -0.2 / -0.3 & 85 / 96.4 & -0.4 / -0.4 \\ 
25--1$\times$5 & 89.6 / 96.9 & 0.4 / 0.1 & 83 / 95.8 & 0.3 / 0.1 \\ 
25--5 & 89.8 / 96.9 & -0.2 / -0.4 & 83.8 / 95.8 & -0.3 / -0.4 \\ 
\midrule
\multicolumn{1}{c}{\multirow{2.5}{*}{Task Schedule}} & \multicolumn{2}{c}{\textbf{Img / Pixel AUROC}} & \multicolumn{2}{c}{\textbf{Image AP / AUPRO}}\\ 
\cmidrule(r){2-3} \cmidrule(r){4-5} 
& \textbf{Performance ($\uparrow$)} & \textbf{FM ($\downarrow$)} & \textbf{Performance ($\uparrow$)} & \textbf{FM ($\downarrow$)} \\
\midrule
UniAD-Oracle & 83.0 / 97.3 & - & 80.9 / 86.7 & - \\ 
PatchCore-Oneclass  & 89.4 / - & - & - & - \\ 
\bottomrule
\end{tabular}
}
\caption{Performance of \ours on Real-IAD dataset.}
\label{table:realiad}
\end{table}

\subsection{Large scale datasets - Real-IAD} 
We evaluate \ours on Real-IAD \cite{realiad}, which is a large-scale dataset. We increase the coreset size $m=250,000$ and leave the hyperparameters same. We follow the official noise-free protocol of Real-IAD with 36,465 normal train images and 114,585 test images composed of 63,256 normal and 51,329 anomalous images with 30 classes in total. We test with various task schedules, as shown in Tab.~\ref{table:realiad} ranging from simple multiple task, complex few tasks, and mix of them. We compare with UniAD-oracle model, which is equivalent as UniAD trained with multi-class setup, and PatchCore-OneClass which are reported in the Real-IAD paper. Despite trained with continual task schedules, \ours outperforms both UniAD-oracle and PatchCore-OneClass models in all reported metrics. The training time took around 65,000 seconds for each model. 

\subsection{Ablation study}
We conduct extensive ablation study on the components of \ours. Specifically, we analyze how the coreset size $m$, sampling ratio $p$, and approximate ratio $q$ affect \ours. Additionally, we show how different backbones affect \ours. Experiments are done on MVTecAD where the result metrics are averaged through all task schedules used in Tab.~\ref{table:mvtec_auroc} and~\ref{table:visa_auroc}. We provide full results of the experiments in the appendix along with experiments on VisA.

\begin{figure}
\centering
\subfigure[\label{fig:coreset_size}]{%
    \includegraphics[width=0.245\linewidth]{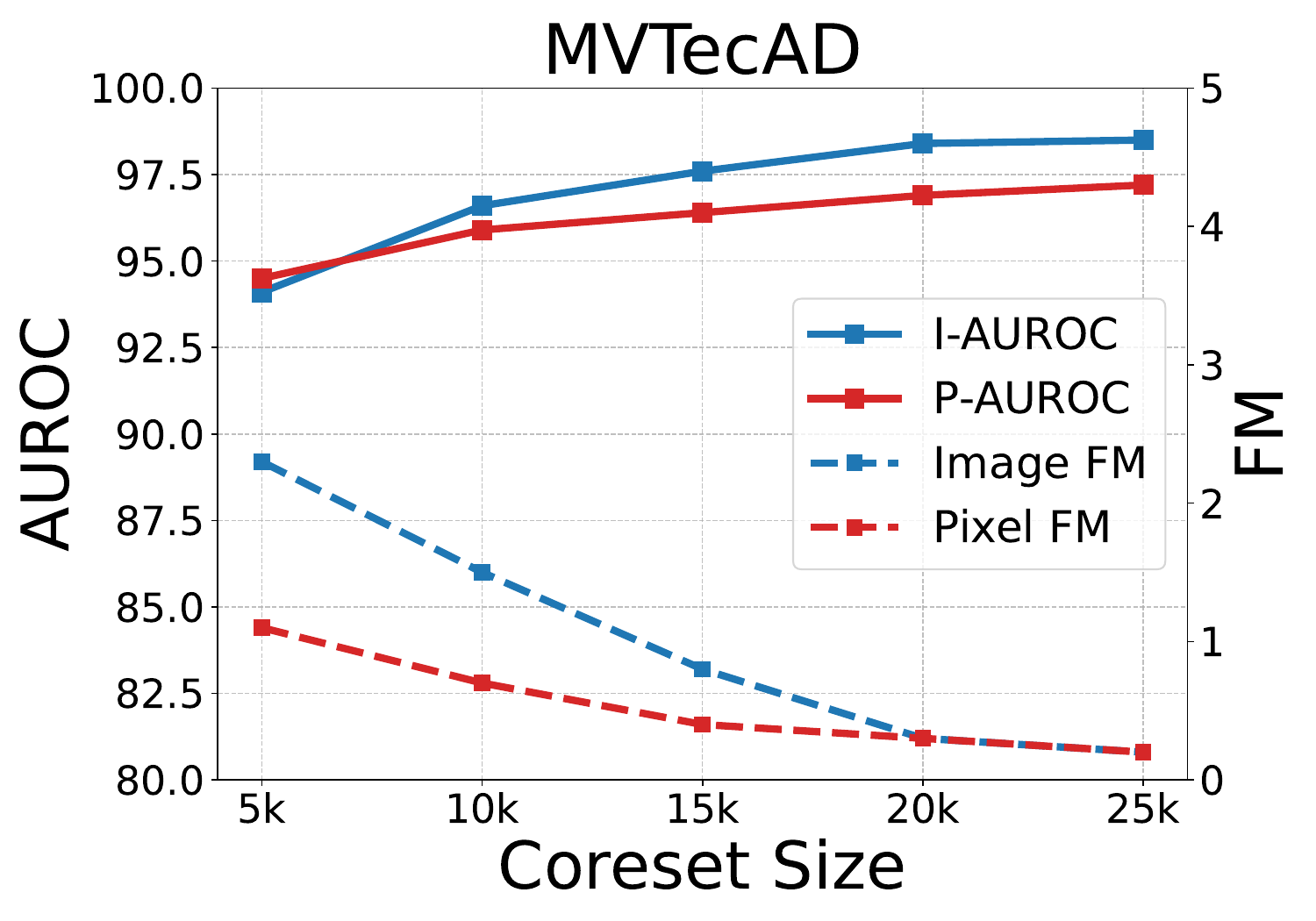}%
}
\subfigure[\label{fig:sampling_ratio}]{%
    \includegraphics[width=0.245\linewidth]{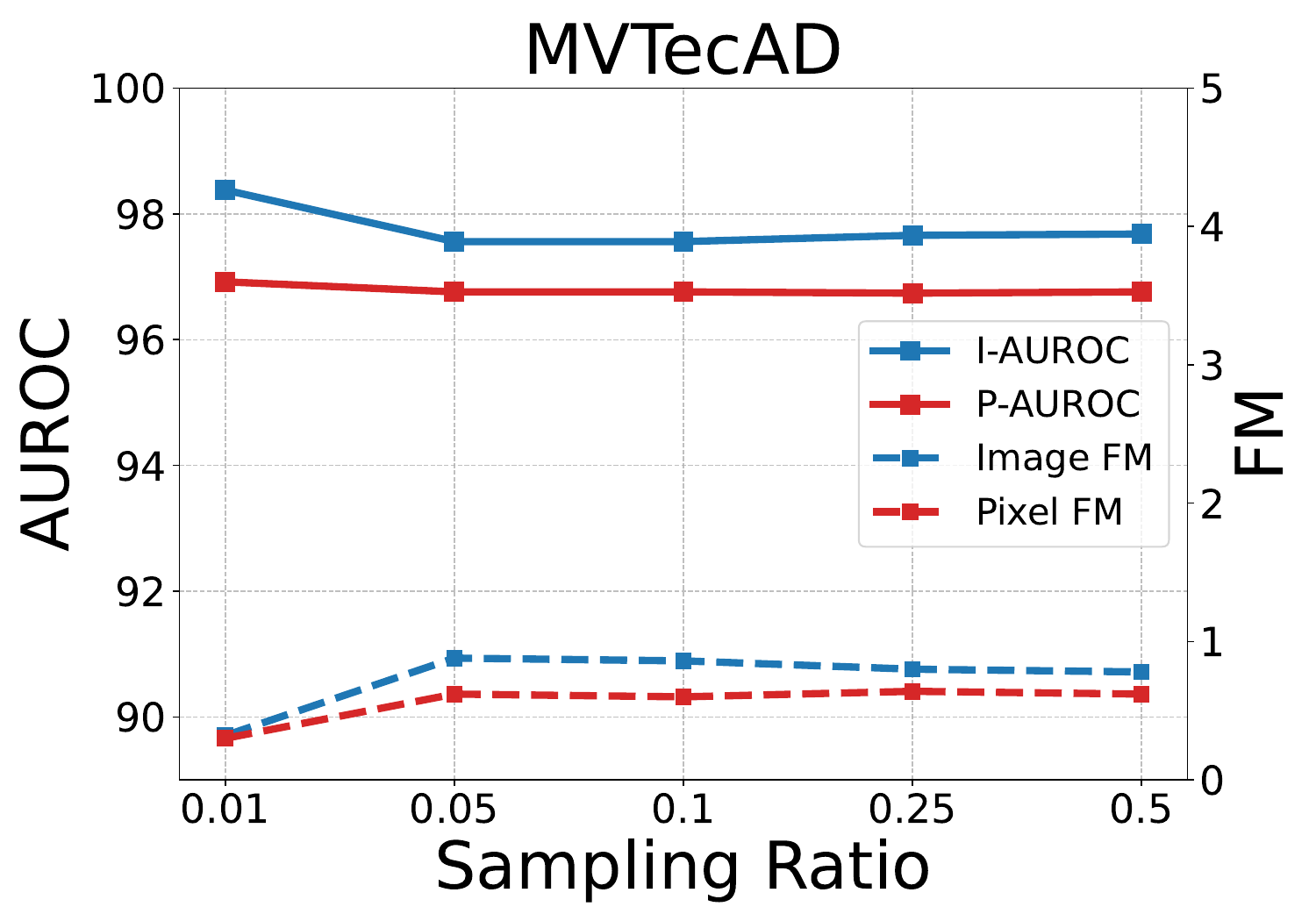}%
}
\subfigure[\label{fig:approximate_ratio}]{%
    \includegraphics[width=0.245\linewidth]{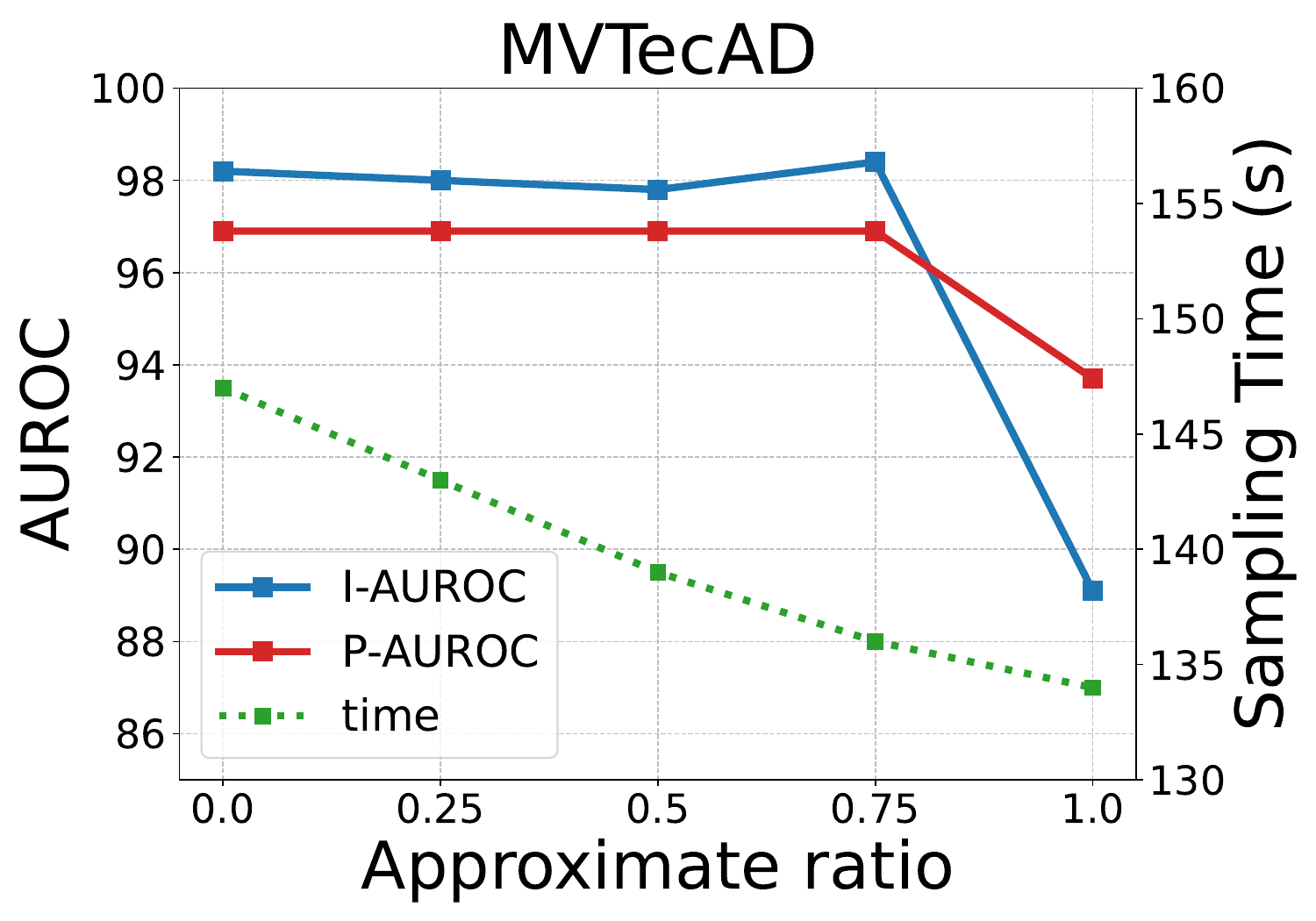}%
}
\subfigure[\label{fig:approximate_ratio_toy}]{%
    \includegraphics[width=0.245\linewidth]{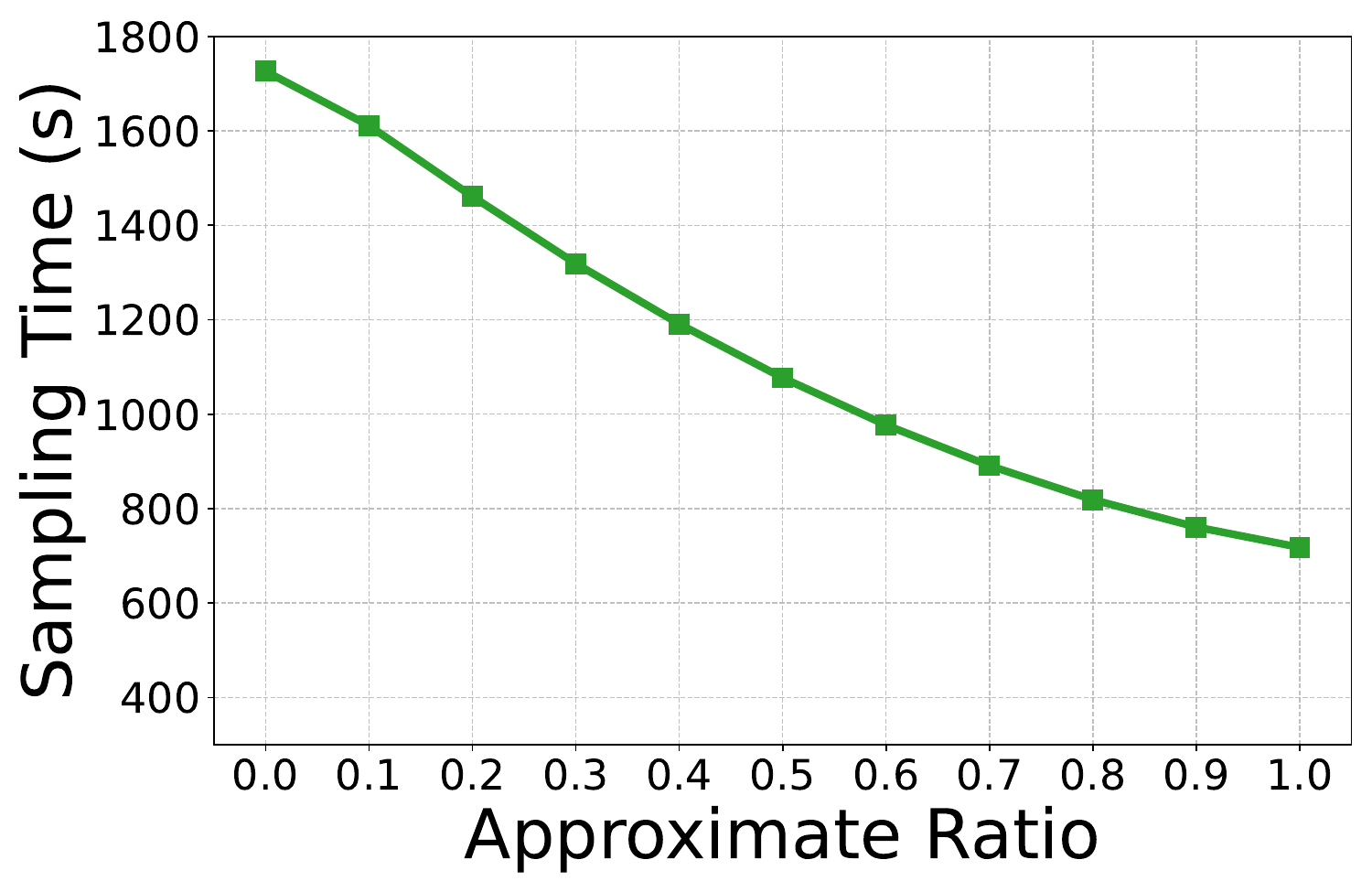}%
}
\caption{Ablation results of \ours on MVTecAD by varying coreset size $m$ (a), sampling ratio $p$ (b), approximation ratio $q$ (c), and effect of approximation ratio $q$ with large coreset size $m$ and sampling ratio $p$ (d).
}
\label{fig:ablation}
\end{figure}

\paragraph{Coreset size $m$} is the most important hyperparameters of \ours, which equivalently works as the capacity of a neural network for network training as it follows the performance-speed tradeoff. Smaller $m$ models are faster but less performative while larger $m$ performs better with the tradeoff of inference speed. Therefore, selecting the proper value for $m$ is crucial. Fig.~\ref{fig:coreset_size} shows how the size of the coreset impacts \ours where performance is proportional to coreset size. However, performance increasement starts to plateu when $m \geq 20000$ which is a good performance-efficiency trade-off point.

\paragraph{Sampling ratio $p$} is crucial where Theorem~\ref{theorem} explains by presenting the importance of sufficient per-task sampling size for effective greedy continuation. It also determines the number of features to sample during greedy expansion, which affects greedy consolidation. As shown in Fig~\ref{fig:sampling_ratio}, the overall performance of varying the sampling ratio from $0.01 \leq p \leq 0.5$ doesn't severely fluctuate. Despite that high $p$ values include more features during greedy expansion, greedy consolidation sorts out the most representative features. As low $p$ values require less time and, we select $p=0.01$. While our experiments were conducted on MVTecAD with sufficient number of samples, higher $p$ might be required for datasets with less samples.

\paragraph{Approximation ratio} $q$ reduces the overall sampling time of \ours by approximating greedy consolidation via nearest-neighbor. In Fig~\ref{fig:approximate_ratio}, we present the overall performance of varying the $q$. The overall performance remains stable in the interval $0 \leq q \leq 0.75$ but drops drastically at $q=1$ where relying on nearest-neighbor ranking alone destroys the coverage structure. The sampling time decreases as $q$ gets larger which is more significant in results on VisA in appendix. As the effect of $q$ is more prominent with larger coreset and higher sampling ratios(more features from greedy expansion) we conduct an additional experiment on MVTecAD with $p=0.2$ and $m=100000$, and vary only $q$ to measure its effect on sampling time. Please note that this experiment focuses on the difference in sampling time based on approximate ratio of \ours regardless of the performance. Fig~\ref{fig:approximate_ratio_toy} shows that the sampling time drops down to half with large $p$, $m$ values with $q=0.75$.

\begin{table}[t]
\centering
\resizebox{0.6\linewidth}{!}{
\begin{tabular}{lcccc}
\toprule
\multicolumn{1}{c}{\multirow{2.5}{*}{Backbone}} & \multicolumn{2}{c}{\textbf{MVTecAD}} & \multicolumn{2}{c}{\textbf{VisA}}   \\ 
\cmidrule(r){2-3} \cmidrule(r){4-5} 
& \textbf{AUROC ($\uparrow$)} & \textbf{FM ($\downarrow$)} & \textbf{AUROC ($\uparrow$)} & \textbf{FM ($\downarrow$)} \\
\midrule
DenseNet201 & 96.9 / 95.4 & 0.8 / 0.4 & 92.7 / 96.4 & 2.2 / 0.4 \\ 
DINOv2 & 97.9 / 93.1 & 0.0 / 0.0 & 93.7 / 95.1 & 0.0 / 0.0 \\ 
EfficientNet-b1 & 96.0 / 94.4 & 0.0 / 0.0 & 89.6 / 95.8 & 0.1 / 0.0 \\ 
WideResNet101 & 98.5 / 97.2 & 0.4 / 0.3 & 94.4 / 97.3 & 1.2 / 0.5 \\ 
\bottomrule
\end{tabular}
}
\caption{Ablation of \ours using different backbones.}
\label{table:abl_bbone}
\end{table}

\paragraph{Different backbones} We validate the effect of varying the backbone and its effect on \ours. Tab.~\ref{table:abl_bbone} presents that the performance is correlated with the representation quality of extracted features. \ours consistently performs well across various backbone architectures, showing that the effectiveness of \ours originates from the greedy-continued coreset construction and the principle holds across different feature extractors, while absolute performance scales with feature quality.

\subsection{Empirical tightness of theory}
We empirically verify the inequality of Eq.~\eqref{eq:inequality} from Sec.~\ref{sec:theorem} by considering the equality case where $\epsilon_o = H(\mathcal{O}, \; \cup_{t=1}^T \mathcal{Z}_t)$, $\epsilon_t = H(S_t, \; \mathcal{Z}_t)$, and $\widehat{\epsilon}_t = H(\mathcal{M}_t, \; S_t \cup \mathcal{M}_{t-1})$. For this experiment, we build a toy dataset by choosing the first 10 classes of MVTecAD and sample 12 images from each class. This setup is due to the computation of the term $H(\mathcal{O}, \; \cup_{t=1}^T \mathcal{Z}_t)$ requires the pair-wise distance between every feature from the dataset and the sampled coreset, which is excessive. Then, we conduct single-class incremental setup for tasks $T\in\{2,3,...,10\}$ with coreset size $m = 150\times T$ and measure the value of each components in the inequality As shown in Tab.~\ref{table:theorem_bound}, the right hand side of Eq.~\eqref{eq:inequality} has higher values than $H(\mathcal{O}, \; \mathcal{M}_T)$ for every $T$.

\begin{table}[t]
\centering
\resizebox{0.6\linewidth}{!}{
\begin{tabular}{l@{\hspace{0.6em}}c@{\hspace{0.6em}}c@{\hspace{0.6em}}c@{\hspace{0.3em}}c}
\toprule
$T$ 
& $H(\mathcal{O}, \; \mathcal{M}_T)$ 
& $\epsilon_o$ 
& $\max_{1 \leq k \leq T} \left( \epsilon_k + \sum_{j=k}^T \widehat{\epsilon}_j \right)$ 
& Right hand side of Eq.~\eqref{eq:inequality} \\  
\midrule
2 & 4.87 & 2.28 & 4.37 & 6.65 \\  
3 & 4.72 & 2.22 & 5.99 & 8.21 \\  
4 & 4.59 & 2.13 & 7.54 & 9.67 \\  
5 & 4.50 & 2.04 & 9.02 & 11.06 \\  
6 & 4.44 & 2.08 & 10.66 & 12.74 \\  
7 & 4.35 & 2.08 & 12.27 & 14.35 \\  
8 & 4.35 & 2.21 & 13.95 & 16.16 \\  
9 & 4.31 & 2.07 & 15.22 & 17.29 \\  
10 & 4.36 & 2.19 & 16.85 & 19.04 \\  
\bottomrule
\end{tabular}
}
\caption{Theoretical tightness of bound shown for different number of tasks $\{2 \leq T \leq 10 \}$ on MVTecAD classes (tasks) with 12 images from each and the value of each of the terms.}
\label{table:theorem_bound}
\end{table}

\subsection{Online learning} 
We evaluate \ours under a more challenging setup, which we term ``online CAD." Unlike previous CAD works that allow multiple data epochs and unrestricted batch size, this protocol introduces stricter, more realistic constraints. On top of the continual task schedules, we limit each data sample to be processed exactly once and limit the batch size to 1. We increase the sampling ratio $p=0.1$ to acquire maximally informative patch features from single samples and leave the rest of hyperparameters the same. Tab.~\ref{table:abl_online_continual} shows that \ours outperforms the SOTA CAD methods under these demanding conditions with minor performance drop compared to the standard CAD protocol shown in Tab.~\ref{table:mvtec_auroc} and~\ref{table:visa_auroc}.

\begin{table}[t]
\centering
\resizebox{.6\linewidth}{!}{
\begin{tabular}{lcccc}
\toprule
\multicolumn{1}{c}{\multirow{2.5}{*}{Method}} & \multicolumn{2}{c}{\textbf{MVTecAD}} & \multicolumn{2}{c}{\textbf{VisA}}   \\ 
\cmidrule(r){2-3} \cmidrule(r){4-5} 
& \textbf{AUROC ($\uparrow$)} & \textbf{FM ($\downarrow$)} & \textbf{AUROC ($\uparrow$)} & \textbf{FM ($\downarrow$)} \\
\midrule
DNE$\;_\text{ACMMM'22}$ \cite{dne} & 69.1 / - & 1.3 / - & 65.0 / - & 3.1 / - \\ 
UCAD$\;_\text{AAAI'24}$ \cite{ucad} & \ul{84.2} / 76.2 & \textbf{-1.1} / \ul{0.3} & \ul{74.6} / 75.8 & \textbf{0.1} / \ul{0.4} \\ 
IUF$\;_\text{ECCV'24}$ \cite{iuf} & 71.0 / \ul{80.5} & \ul{0.4} / \textbf{0.1} & 60.3 / 85.0 & 0.9 / 1.7 \\ 
CDAD$\;_\text{CVPR'25}$ \cite{cdad} & 66.2 / 80.4 & 2.0 / \textbf{0.1} & 58.9 / \ul{87.5} & \ul{0.2} / \textbf{0.3} \\ 
\midrule
\rowcolor{red!25}
\textbf{\ourso} & \textbf{97.2 / 96.7} & 1.5 / 0.9 & \textbf{94.1 / 97.2} & 1.3 / 0.8 \\ 
\bottomrule
\end{tabular}
}
\caption{The results on online continual anomaly detection setup.}
\label{table:abl_online_continual}
\end{table}

\begin{table*}[t]
\centering
\resizebox{.6\linewidth}{!}{
\begin{tabular}{lcccc}
\toprule
\multicolumn{1}{c}{\multirow{2.5}{*}{Method}} & \multicolumn{2}{c}{\textbf{MVTecAD}} & \multicolumn{2}{c}{\textbf{VisA}}   \\ 
\cmidrule(r){2-3} \cmidrule(r){4-5} 
& \textbf{AUROC ($\uparrow$)} & \textbf{FM ($\downarrow$)} & \textbf{AUROC ($\uparrow$)} & \textbf{FM ($\downarrow$)} \\
\midrule
DNE$\;_\text{ACMMM'22}$ \cite{dne} & 82.6 / - & 1.9 / - & 61.7 / - & 7.4 / - \\ 
UCAD$\;_\text{AAAI'24}$ \cite{ucad} & 90.3 / 91.2 & 4.9 / 3.1 & 80.7 / 89.8 & 5.5 / 2.9 \\ 
IUF$\;_\text{ECCV'24}$ \cite{iuf} & 87.7 / 91.2 & 5.7 / 4.2 & 78.0 / 92.9 & 6.5 / 5.1 \\ 
CDAD$\;_\text{CVPR'25}$ \cite{cdad} & 90.1 / 93.9 & 3.3 / 2.7 & 83.9 / 96.9 & 3.8 / 1.9 \\ 
\midrule
PatchCore Oracle & \textbf{98.4} / \textbf{97.2} & \textbf{0.0} / \textbf{0.0} & \textbf{95.1} / \textbf{97.6} & \textbf{-0.1} / \ul{0.0} \\
PatchCore Oracle Online & \ul{91.2} / 90.1 & \textbf{0.0} / \textbf{0.0} & 77.9 / 85.2 & \ul{0.0} / \textbf{-0.3} \\ 
\midrule
\rowcolor{red!25}
\textbf{\ourso} & \textbf{98.4} / \ul{96.9} & \ul{0.3} / \ul{0.3} & \ul{95.0} / \ul{97.2} & 0.8 / 0.5 \\ 
\bottomrule
\end{tabular}
}
\caption{Comparison with PatchCore-Oracle.}
\label{table:abl_patchcore_oracle}
\end{table*}

\subsection{\ours and PatchCore} 
PatchCore involves a single offline greedy sampling from the whole feature set which is not inherently designed for continual learning and naively accumulating new features leads to memory overflow. To establish a comparable coreset based baseline, we propose PatchCore-Oracle where the total number of tasks is given as oracle information to uniformly divide the memory budget across all tasks. As shown in Tab.~\ref{table:abl_patchcore_oracle}, \ours performs comparably in a standard continual setting even against this advantaged baseline. However, in the online CAD setup where batch size is 1, PatchCore-Oracle's performance degrades due to the nondiversified feature set to sample at each step which shows the limitation: the lack of the mechanism refering to the current coreset for selecting new features from incoming tasks while \ours, is explicitly designed for such dynamic updates. 

Furthermore, our goal is to construct a coreset under a continual setup that approximates that of PatchCore which is sampled from all features. Therefore, we analyze the sampled coreset of \ours, PatchCore, and a random sampling. Tab.~\ref{table:coreset_wasserstein} shows the results where \ours has more overlapping samples, lower avearge minimum distance for each sample, lower sliced Wasserstein distance, and lower Hausdorff distance than random sampling, which quantitatively confirms that the coreset of \ours is more similar to that of PatchCore-Oracle.

\begin{table}[t]
\centering
\resizebox{0.65\linewidth}{!}{
\begin{tabular}{r@{\hspace{0.6em}}c@{\hspace{0.6em}}c@{\hspace{0.6em}}c@{\hspace{0.3em}}c}
\toprule
\multirow{2}{*}{\shortstack{Sampling\\Method}} & \textbf{\multirow{2}{*}{\shortstack{Overlapping\\Samples ($\uparrow$)}}} & \textbf{\multirow{2}{*}{\shortstack{Average Minimum\\Distance ($\downarrow$)}}} & \textbf{\multirow{2}{*}{\shortstack{Sliced Wasserstein\\Distance ($\downarrow$)}}} & \textbf{\multirow{2}{*}{\shortstack{Hausdorff\\Distance ($\downarrow$)}}}  \\
\\
\midrule
Random &  356 & 1.17 & 0.0260 & 2.55 \\
\rowcolor{red!25}
\textbf{MGS} & 3933 & 1.04 & 0.0014 & 1.93 \\
\bottomrule
\end{tabular}
}
\caption{Comparing of coresets of \ours and random sampling with respect to PatchCore model.}
\label{table:coreset_wasserstein}
\end{table}

\subsection{Qualitative analysis} 
We compare the segmentation qualities of CAD methods on MVTecAD with different task schedules. To visualize forgetting in the extreme case, we show the samples from the oldest task. As shown in Fig.~\ref{fig:samples}, UCAD fails with complex tasks such as (14--1, 10--5) and IUF and CDAD fail for long-ranged tasks ($1\times14$) while \ours shows stable fine-grained localization quality. We provide more qualitative results of each task schedule in the appendix.

\begin{figure}[t]
\centering
\includegraphics[width=.45\linewidth]{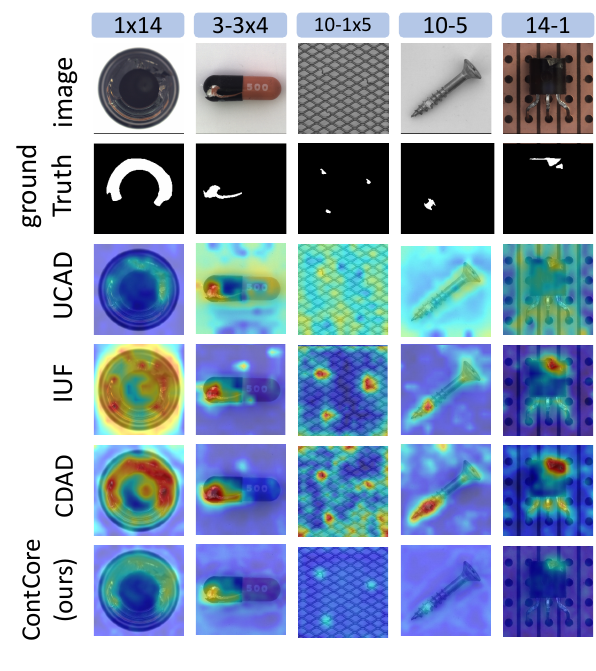}
\caption{
Anomaly segmentation qualitative comparison of \ours and the SOTA methods by different task schedules on MVTecAD. All samples are from the first task.
}
\label{fig:samples}
\vspace{-.5em}
\end{figure}

\section{Conclusion}
\label{sec:conclusion}
Greedy sampling selects maximally separated points, producing a compact yet representative summary of normal data. Since anomaly detection relies on measuring distance from normality, a representative coreset is essential for reliable detection. The challenge for continual AD is maintaining representativeness as tasks accumulate under memory constraints. We observed that continued greedy sampling effectively preserves representativeness under strict memory limits, and provided theoretical justification showing the greedy-continued coreset approximates the oracle coreset within a bounded gap. We instantiated this principle in \ours, which constructs a greedy-continued coreset through greedy expansion and consolidation. Unlike neural methods susceptible to catastrophic forgetting or naive coreset methods requiring unbounded memory, \ours maintains fixed memory with theoretical guarantees. \ours achieves state-of-the-art results across 11 continual task schedules on MVTecAD and VisA while maintaining efficient training and inference. The approach extends to online continual AD settings where prior methods fail significantly, demonstrating robustness beyond specific algorithmic choices.

\paragraph{Acknowledgments}
Yoon Gyo Jung and Octavia Camps were supported by ONR grant N00014-21-1-2431. Kuan-Chuan Peng was exclusively supported by Mitsubishi Electric Research Laboratories. The views and conclusions contained in this document are those of the authors and should not be interpreted as necessarily representing the official policies, either expressed or implied, of the ONR.


\bibliography{refs}

\clearpage
\renewcommand{\thepage}{S\arabic{page}}
\setcounter{page}{1}

\setcounter{theorem}{0}

\section{Theoretical Analysis}
\label{sec:theory}

\begin{theorem}
\label{thm:meta_greedy_error}
Suppose
\begin{equation}
H(\cO, \cup_{t=1}^T \cZ_t) \leq \epsilon_o,
\end{equation}
and 
\begin{equation}
H(S_t, \cZ_t) \leq \epsilon_t,
\quad
H(\cM_t, S_t \cup \cM_{t-1}) \leq \widehat{\epsilon}_t
\end{equation}
for all $t=1,\dots,T$ with $S_1 = \cM_1$, $\cM_0 = \emptyset$, and $\widehat{\epsilon}_0 = 0$,
where $\cO = \cG_{n_0}(\cup_{t=1}^T \cZ_t)$, $S_t = \cG_{n_t}(\cZ_t; \cM_{t-1})$, and $\cM_t = \cG_{m_t}(S_t \cup \cM_{t-1})$.
Then, we have
\begin{equation}
H(\cO, \cM_T)
\leq
\epsilon_o +
\max_{1 \leq k \leq T} \left( \epsilon_k + \sum_{j=k}^T \widehat{\epsilon}_j \right)
\end{equation}
for sufficiently large $n_o \leq | \cup_{t=1}^T \cZ_t |$, $n_t \leq |\cZ_t|$, and $m_t \leq |S_t \cup \cM_{t-1}|$
\end{theorem}

We prove the above main theorem step-by-step as follows:

\subsection{Definitions}
The greedy sampling operation on a set $\cZ$ of embeddings and on a base set $B$ is defined as follows:
\begin{equation}
\cG_n(\cZ; B) 
= \{
z_{i} \in \cZ: 
z_{i} 
= 
\underset{z \in \cZ \setminus \{z_1, \dots, z_{i-1} \}}{\arg \max}
d(z, \{z_1, \dots, z_{i-1} \} \cup B)
\}
\end{equation}
for $2 \leq i \leq n$. The initial point $z_1$ is given as  $z_1 = \arg \max_{z \in \cZ}d(z, B)$. We denote $\cG(\cZ) = \cG(\cZ; \emptyset)$. 

Define 
\begin{equation}
d(\cZ_1, \cZ_2) = \min_{z_1 \in \cZ_1} \min_{z_2 \in \cZ_2} d(z_1, z_2),
\end{equation}
Then  Hausdorff distance $H$ is defined by
\begin{equation}
H(S_1, S_2) = \max( \max_{z_1 \in S_1} d(z_1, S_2), \max_{z_2 \in S_2} d(z_2, S_1)).
\end{equation}

\subsection{Derivations}

We first state Triangle Inequality and Subadditivity of Hausdorff distance \cite{deza2009encyclopedia,burago2001course,gromov1999metric}, which we will freely use throughout the derivations:

\begin{proposition}[Triangle Inequality]
Let $A$, $B$, and $C$ be subsets of the same metric space with metric $d$. Then,
\begin{equation}
H(A, C) \leq H(A,B) + H(B,C).    
\end{equation}
\end{proposition}

\begin{proposition}[Subadditivity]
\begin{equation}
H(A \cup B, C  )
\leq
\max
\left(
H(A, C), H(B, C)
\right)
\end{equation}
\end{proposition}
Note that the subadditivity trivially implicates
\begin{equation}
H(A \cup B, C \cup D)
\leq 
\max
\left(
H(A, C), H(B, D)
\right)
\end{equation}
as $C \subseteq C \cup D$ and same for $D$.

\begin{theorem}
\label{thm:inductive_meta_greedy_bound}
Suppose $H(\cO, \cup_{t=1}^T \cZ_t) \leq \epsilon_o$, and for all $t=1,\dots,T$ we have $H(S_t, \cZ_t) \leq \epsilon_t$ with $S_1 = \cM_1$. Suppose $H(\cM_t, S_t \cup \cM_{t-1}) \leq \widehat{\epsilon}_t$ for $t=1, \dots, T$ with $\cM_0 = \emptyset$ and $\widehat{\epsilon}_0 = 0$. Then, we have
\begin{equation}
H(\cO, \cM_T) \leq \epsilon_o + \delta_T
\end{equation}
where 
\begin{equation}
\delta_t = 
\widehat{\epsilon}_t + \max(\epsilon_t, \delta_{t-1})
\end{equation}
for $t=1, \dots, T$ with $\delta_0 = 0$
\end{theorem}

\begin{proof}
We will prove by induction that
$$
H\!\biggl(\,\bigcup_{k=1}^t \cZ_k,\; \cM_t\biggr) \;\;\le\;\;\delta_t
  \quad
\text{for each}
$$
for each $t=1,\dots,T$ where 
$$
  \delta_t \;=\; \widehat{\epsilon}_t \;+\;\max\bigl(\epsilon_t,\;\delta_{t-1}\bigr),
  \qquad
  \delta_0 = 0.
$$
Once this is established for $t=T$, we can then combine it with
$$
  H\!\Bigl(\cO,\;\bigcup_{k=1}^T \cZ_k\Bigr) 
  \;\;\le\;\;\epsilon_o
$$
to conclude the desired by the triangle inequality.

\textbf{Base Case} $t=1$. Recall $S_1 = \cM_1$. Then
$$
  H(\cZ_1, \cM_1)
  \;\;\le\;\;
  H(\cZ_1, S_1) \;+\; H(S_1, \cM_1)
  \;\;=\;\;
  \epsilon_1 \;+\; \widehat{\epsilon}_1
  \;\;=\;\;
  \delta_1.
$$
So the claim holds for $t=1$.

\textbf{Inductive Step}. Suppose for some $t-1 \ge 1$ we have
$$
  H\!\Bigl(\bigcup_{k=1}^{t-1} \cZ_k,\; \cM_{t-1}\Bigr) \;\;\le\;\; \delta_{t-1}.
$$
By the triangle inequality and the assumption $H(\cM_{t-1} \cup S_t, \cM_t) \leq \widehat{\epsilon}_t$,
$$
  H\Bigl(\!\bigcup_{k=1}^t \cZ_k,\; \cM_t\Bigr)
  \;\;\le\;\;
  H\Bigl(\!\bigcup_{k=1}^{t-1} \cZ_k \cup \cZ_t,\; \cM_{t-1}\cup S_t\Bigr)
  \;+\;
  \widehat{\epsilon}_t,
$$
where
$$
  H\Bigl(\!\bigcup_{k=1}^{t-1} \cZ_k\cup \cZ_t,\; \cM_{t-1}\cup S_t\Bigr)
  \;\;\le\;\;
  \max\Bigl\{
    H\bigl(\!\bigcup_{k=1}^{t-1} \cZ_k,\; \cM_{t-1}\bigr),
    \;H(\cZ_t, S_t)
  \Bigr\}
\; \leq \; 
\max( \delta_{t-1}, \epsilon_t )
$$
Putting it all together:
$$
  H\Bigl(\!\bigcup_{k=1}^t \cZ_k,\; \cM_t\Bigr)
  \;\;\le\;\;
  \max(\delta_{t-1},\;\epsilon_t)
  \;+\;
  \widehat{\epsilon}_t
  \;=\;
  \delta_t.
$$
\end{proof}

\begin{proposition}
\label{prop:inductive_error}
Consider the inductive sequence 
\[
\delta_0 = 0
\quad\text{and}\quad
\delta_t \;=\; \widehat{\epsilon}_t \;+\;\max\bigl(\epsilon_t,\;\delta_{t-1}\bigr)
\quad\text{for}\quad t\ge1.
\]
A direct ``unrolling'' of this recursion reveals the closed‐form expression
\begin{equation}
\label{eq:delta_closed_form}
\delta_t 
\;=\;
\max_{1 \,\le\, k \,\le\, t}
\biggl(\,\epsilon_k \;+\; \sum_{j=k}^{t} \widehat{\epsilon}_j\biggr).
\end{equation}
\end{proposition}

\begin{proof}
First check $t=1$. In that case,
\[
\delta_1 
\;=\; 
\widehat{\epsilon}_1 + \max(\epsilon_1,\,\delta_0)
\;=\;
\widehat{\epsilon}_1 + \epsilon_1
\;=\;
\epsilon_1 \;+\; \sum_{j=1}^1 \widehat{\epsilon}_j
\;=\;
\max_{1 \,\le\, k \,\le\, 1}
\bigl(\epsilon_k + \widehat{\epsilon}_k \bigr).
\]
Assume for some $t\ge1$ that
\[
\delta_{t-1}
\;=\;
\max_{1 \,\le\, k \,\le\, t-1}
\biggl(\,\epsilon_k \;+\; \sum_{j=k}^{t-1} \widehat{\epsilon}_j\biggr).
\]
Then
\[
\delta_t 
\;=\;
\widehat{\epsilon}_t + \max\bigl(\epsilon_t,\;\delta_{t-1}\bigr)
\;=\;
\max\Bigl(\epsilon_t + \widehat{\epsilon}_t,\;
\delta_{t-1} \;+\; \widehat{\epsilon}_t\Bigr).
\]
Inserting the inductive hypothesis into the second term,
\[
\delta_{t-1} + \widehat{\epsilon}_t
\;=\;
\max_{1 \,\le\, k \,\le\, t-1}
\biggl(\,\epsilon_k \;+\; \sum_{j=k}^{t-1} \widehat{\epsilon}_j\biggr)
\;+\;
\widehat{\epsilon}_t
\;=\;
\max_{1 \,\le\, k \,\le\, t-1}
\biggl(\,\epsilon_k \;+\; \sum_{j=k}^{t} \widehat{\epsilon}_j\biggr).
\]
Hence
\[
\delta_t
\;=\;
\max\Bigl(\,
\epsilon_t + \widehat{\epsilon}_t,\;
\max_{1 \,\le\, k \,\le\, t-1}
\bigl(\epsilon_k + \sum_{j=k}^{t} \widehat{\epsilon}_j\bigr)
\Bigr)
\;=\;
\max_{1 \,\le\, k \,\le\, t}
\biggl(\epsilon_k + \sum_{j=k}^{t} \widehat{\epsilon}_j\biggr).
\]
This completes the inductive argument and establishes \eqref{eq:delta_closed_form}.
\end{proof}

\begin{proof}[Proof of Theorem \ref{thm:meta_greedy_error}]
By Theorem \ref{thm:inductive_meta_greedy_bound}, 
\begin{equation}
H(\cO, \; \cM_T) \leq \epsilon_o + \delta_T
\end{equation}
where 
\begin{equation}
\delta_t = 
\widehat{\epsilon}_t + \max(\epsilon_t, \delta_{t-1})
\end{equation}
for $t=1, \dots, T$ with $\delta_0 = 0$. Therefore, by Proposition \ref{prop:inductive_error}, we complete the proof.
\end{proof}

\newpage

\section{Additional results}

\begin{figure*}[ht]
\centering
\subfigure[]{%
    \includegraphics[width=0.32\linewidth]{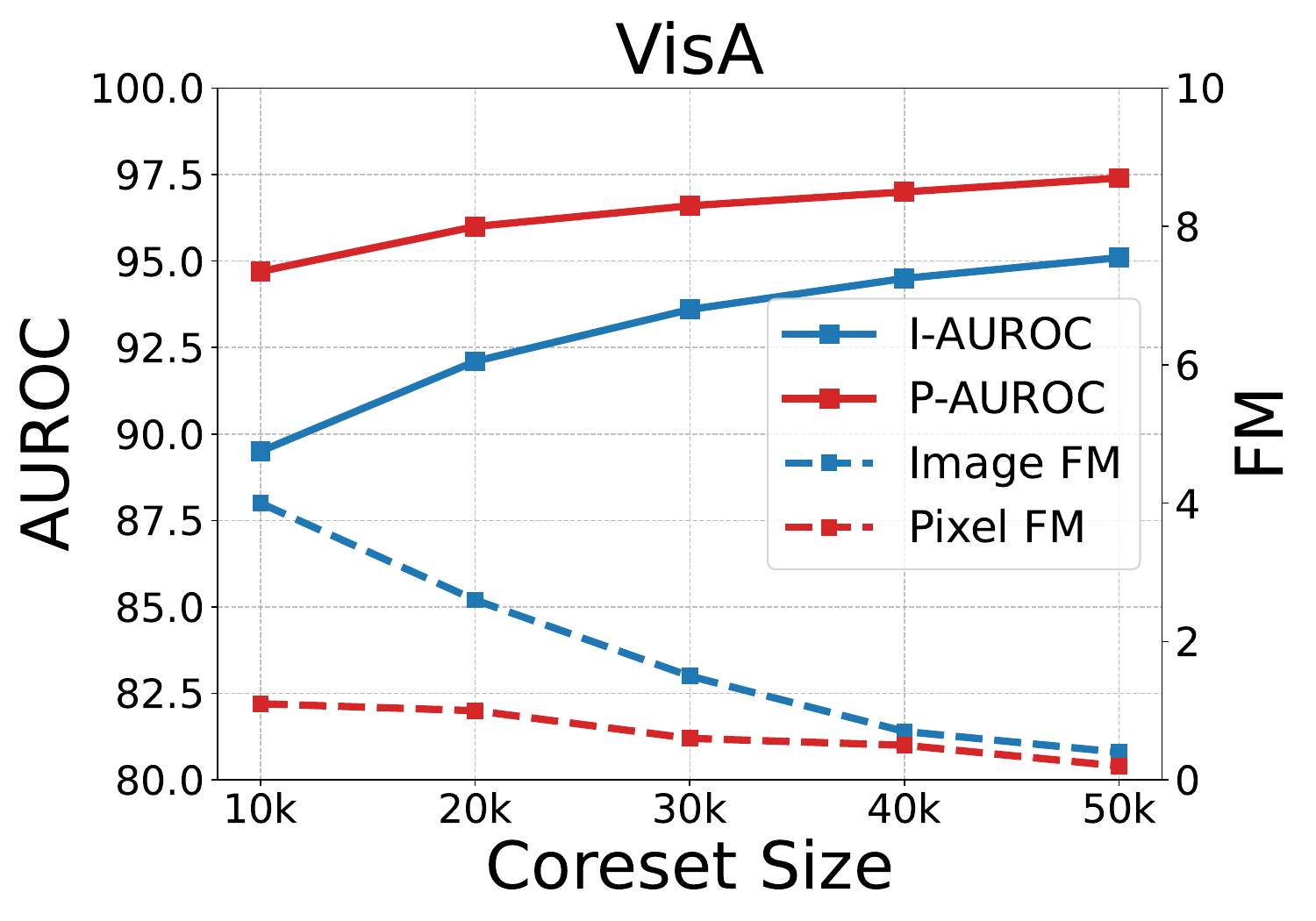}%
}
\subfigure[]{%
    \includegraphics[width=0.32\linewidth]{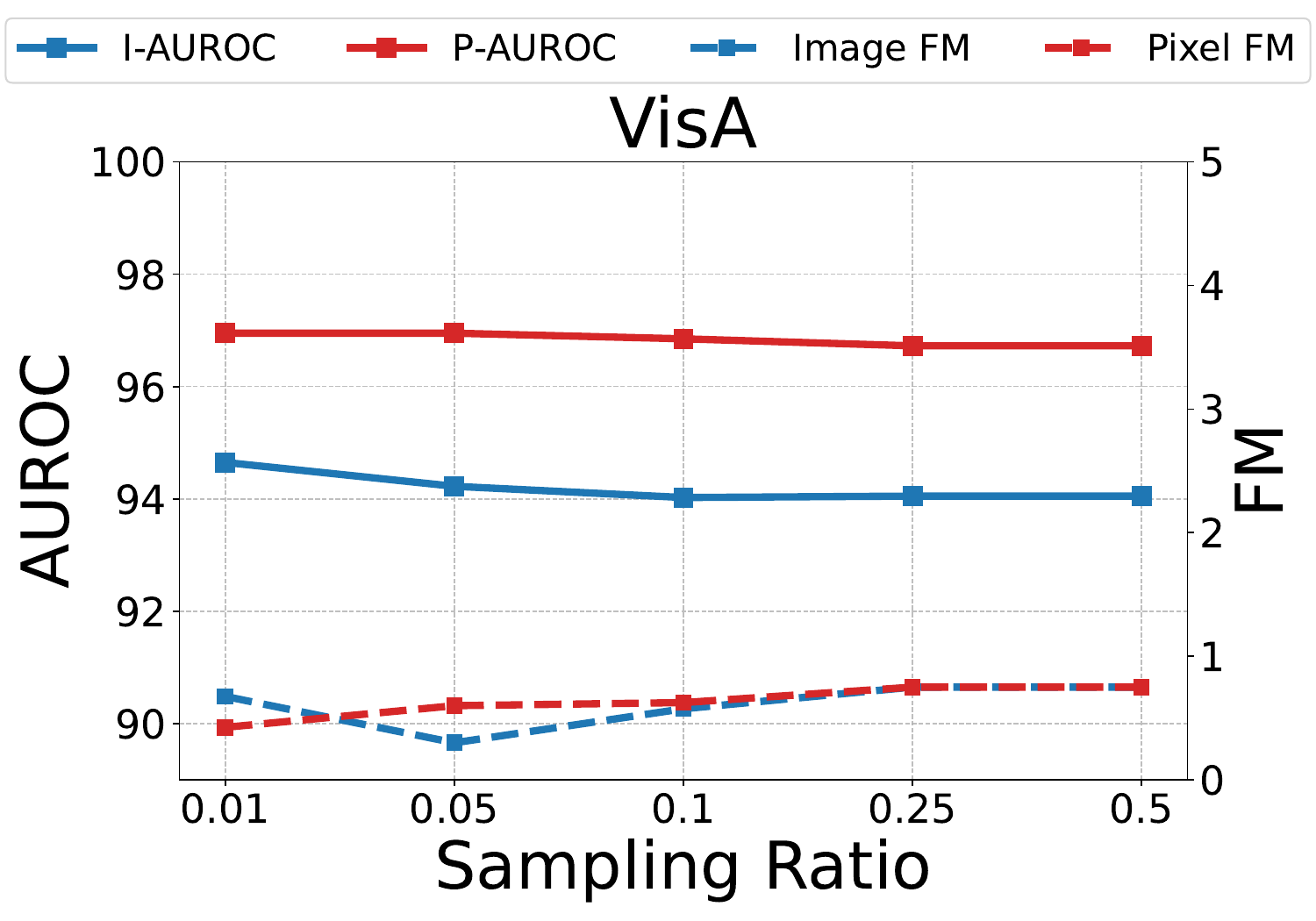}%
}
\subfigure[]{%
    \includegraphics[width=0.32\linewidth]{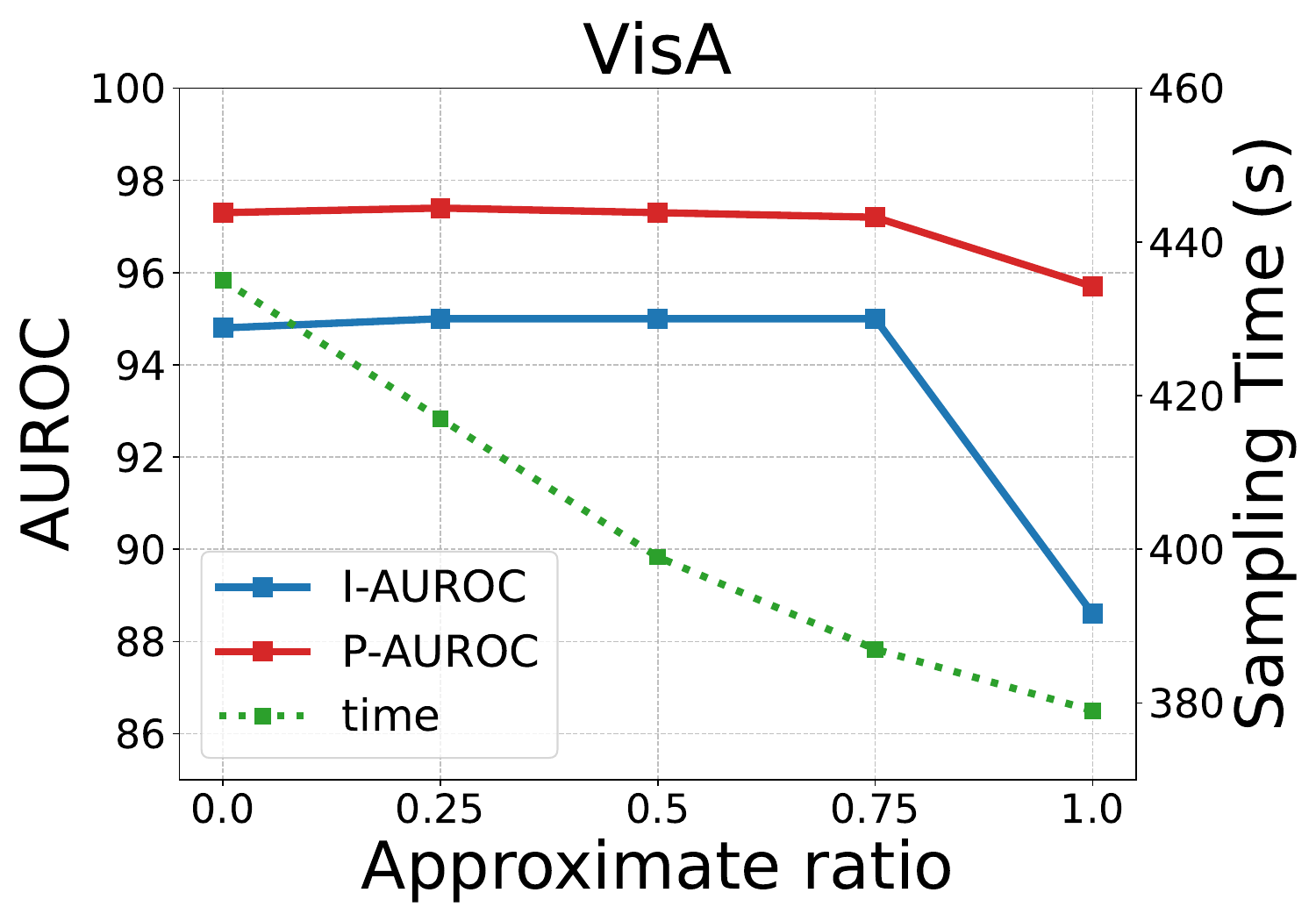}%
}
\caption{
Ablation results of \ours on VisA by varying coreset size $m$ (left), sampling ratio $p$ (middle), approximation ratio $q$ (right).
}
\label{fig:ablation_visa}
\end{figure*}

\begin{table*}[t]
\centering
\caption{The results of MVTecAD dataset on few step incremental schedules where we report task mean average precision (AP) ($\uparrow$) and forgetting measurement (FM) ($\downarrow$). The best and second best results are marked in \textbf{bold} and \ul{underlined}, respectively.
}
\label{table:mvtec_ap}
\resizebox{0.9\linewidth}{!}{
\begin{tabular}{r@{\hspace{2mm}}c@{\hspace{2mm}}c@{}c@{\hspace{2mm}}c@{}c@{\hspace{2mm}}c@{}c@{\hspace{2mm}}c@{}c@{\hspace{2mm}}c}
\toprule
\multicolumn{1}{c}{\multirow{2.5}{*}{Method}} 
& \multicolumn{2}{c}{\textbf{1 × 15 with 15 Steps}} 
& \multicolumn{2}{c}{\textbf{3 × 5 with 5 Steps}} 
& \multicolumn{2}{c}{\textbf{10 -- 1 × 5 with 6 Steps}} 
& \multicolumn{2}{c}{\textbf{10 -- 5 with 2 Steps}} 
& \multicolumn{2}{c}{\textbf{14 -- 1 with 2 Steps}} 
\\ 
\cmidrule(r){2-3} \cmidrule(r){4-5} \cmidrule(r){6-7} \cmidrule(r){8-9} \cmidrule(r){10-11}
& \textbf{AP ($\uparrow$)} & \textbf{FM ($\downarrow$)} & \textbf{AP ($\uparrow$)} & \textbf{FM ($\downarrow$)} & \textbf{AP ($\uparrow$)} & \textbf{FM ($\downarrow$)} & \textbf{AP ($\uparrow$)}& \textbf{FM ($\downarrow$)} & \textbf{AP ($\uparrow$)}& \textbf{FM ($\downarrow$)}\\
\midrule
DNE$\;_\text{ACMMM'22}$ \cite{dne}  
& 75.9 & 10.2 
& 76.4 & 4.4 
& 71.4 & 9.8 
& 77.0 & 2.6 
& 76.0 & 1.9 
\\ 
UCAD$\;_\text{AAAI'24}$ \cite{ucad}
& \ul{95.3} & \textbf{0.0} 
& 93.1  &  2.0 
& 95.5 &  \ul{0.1} 
& 95.0& 1.0
& 95.8  & 0.3  
\\
IUF$\;_\text{ECCV'24}$ \cite{iuf} 
& 84.1 & 2.0 
& 91.1  & 2.9 
& 95.3  & 0.2 
& 95.4  & 1.9  
& 97.8 & 0.3 
\\
CDAD$\;_\text{CVPR'25}$ \cite{cdad} 
& 86.3 & 8.2 
& \ul{95.8} & \ul{1.9} 
& \ul{98.4} & \textbf{0.0} 
& \ul{98.3} & \ul{0.6} 
& \ul{98.4} & \ul{0.1} 
\\
\midrule
\rowcolor{red!25}
\ourso 
& \textbf{99.3} & \ul{0.3} 
& \textbf{99.3}  & \textbf{0.2} 
& \textbf{99.7} & \textbf{0.0} 
& \textbf{99.5}  & \textbf{0.1}  
& \textbf{99.4}  & \textbf{0.0} 
\\ 
\bottomrule
\end{tabular}
}

\end{table*}

\begin{table*}[t]
\centering
\caption{The results of MVTecAD dataset on few step incremental schedules where we report task average PRO ($\uparrow$) and forgetting measurement (FM) ($\downarrow$). The best and second best results are marked in \textbf{bold} and \ul{underlined}, respectively.
}
\label{table:mvtec_pro}
\resizebox{0.9\linewidth}{!}{
\begin{tabular}{r@{\hspace{2mm}}c@{}c@{}c@{}c@{}c@{}c@{}c@{}c@{}c@{}c}
\toprule
\multicolumn{1}{c}{\multirow{2.5}{*}{Method}} 
& \multicolumn{2}{c}{\textbf{1 × 15 with 15 Steps}} 
& \multicolumn{2}{c}{\textbf{3 × 5 with 5 Steps}} 
& \multicolumn{2}{c}{\textbf{10 -- 1 × 5 with 6 Steps}} 
& \multicolumn{2}{c}{\textbf{10 -- 5 with 2 Steps}} 
& \multicolumn{2}{c}{\textbf{14 -- 1 with 2 Steps}} 
\\ 
\cmidrule(r){2-3} \cmidrule(r){4-5} \cmidrule(r){6-7} \cmidrule(r){8-9} \cmidrule(r){10-11}
& \textbf{AUPRO ($\uparrow$)} & \textbf{FM ($\downarrow$)} & \textbf{AUPRO ($\uparrow$)} & \textbf{FM ($\downarrow$)} & \textbf{AUPRO ($\uparrow$)} & \textbf{FM ($\downarrow$)} & \textbf{AUPRO ($\uparrow$)}& \textbf{FM ($\downarrow$)} & \textbf{AUPRO ($\uparrow$)}& \textbf{FM ($\downarrow$)}\\
\midrule
\midrule 
UCAD$\;_\text{AAAI'24}$ \cite{ucad}
& \ul{87.6} & \textbf{-0.8}    
& {71.1}    & {7.5}
& {80.8} & \ul{1.2} 
& 80.7& 2.9    
& 86.3& 1.2 
\\
IUF$\;_\text{ECCV'24}$ \cite{iuf} 
& 78.7 & 6.8
& 72.9 & 5.8 
& 84.3 & 2.4 
& 85.0  & 3.2
& {88.6}    & 0.6    
\\
CDAD$\;_\text{CVPR'25}$ \cite{cdad} 
& 68.5 & 22.1  
& \ul{83.8} &  \ul{4.1} 
& \ul{89.2} & \ul{1.2} 
& \ul{88.9} & \ul{2.6} 
& \ul{89.8} & \ul{0.5} 
\\ 
\midrule
\rowcolor{red!25}
\ourso 
& \textbf{99.3} & \ul{0.3} 
& \textbf{96.3}  & \textbf{0.2}  
& \textbf{99.7} & \textbf{0.0} 
& \textbf{96.0}  & \textbf{0.2}  
& \textbf{96.5}  & \textbf{0.0} 
\\

\bottomrule
\end{tabular}
}
\end{table*}

\begin{table*}[t]
\centering
\caption{The results of VisA dataset on few step incremental schedules where we report task mean average precision (AP) ($\uparrow$) and forgetting measurement (FM) ($\downarrow$). The best and second best results are marked in \textbf{bold} and \ul{underlined}, respectively.
}
\label{table:visa_ap}
\resizebox{0.9\linewidth}{!}{
\begin{tabular}{rcccccccc}
\toprule
\multicolumn{1}{c}{\multirow{2.5}{*}{Method}} 
& \multicolumn{2}{c}{\textbf{1 × 12 with 12 Steps}}  
& \multicolumn{2}{c}{\textbf{8 -- 1 × 4 with 5 Stepss}} 
& \multicolumn{2}{c}{\textbf{8 -- 4 with 2 Steps}} 
& \multicolumn{2}{c}{\textbf{11 -- 1 with 2 Steps}} 
\\ 
\cmidrule(r){2-3} \cmidrule(r){4-5} \cmidrule(r){6-7} \cmidrule(r){8-9} 
& \textbf{AP ($\uparrow$)} & \textbf{FM ($\downarrow$)} & \textbf{AP ($\uparrow$)} & \textbf{FM ($\downarrow$)} & \textbf{AP ($\uparrow$)} & \textbf{FM ($\downarrow$)} & \textbf{AP ($\uparrow$)}& \textbf{FM ($\downarrow$)}\\
\midrule
DNE$\;_\text{ACMMM'22}$ \cite{dne} 
& 65.1 & 11.6
& 66.1 & \textbf{1.3} 
& 67.3 & 9.8 
& 70.8 & 1.6 
\\

UCAD$\;_\text{AAAI'24}$ \cite{ucad} 
& 81.7 & \textbf{0.0} 
& {82.9} & 2.2 
& {83.2} & 5.2 
& {88.1} & 0.3 
\\

IUF$\;_\text{ECCV'24}$ \cite{iuf} 
& \ul{88.8} & 1.2 
& 83.0 & 6.9 
& {83.4} & 7.5 
& \ul{91.6}& \ul{-0.0}
\\

CDAD$\;_\text{CVPR'25}$ \cite{cdad}
& 67.6 & 14.0 
& \ul{84.7}& \ul{1.8} 
& \ul{85.3}& \ul{3.1} 
& 89.4 & \textbf{-0.8} 
\\ 

\midrule
\rowcolor{red!25}
\ourso 
& \textbf{96.5} & \ul{0.1}  
& \textbf{94.1} & \ul{1.8} 
& \textbf{96.7} & \textbf{0.4} 
& \textbf{94.8} & 0.6 
\\ 
\bottomrule
\end{tabular}
}

\end{table*}

\begin{table*}[t]
\centering
\caption{The results of VisA dataset on few step incremental schedules where we report task average PRO ($\uparrow$) and forgetting measurement (FM) ($\downarrow$). The best and second best results are marked in \textbf{bold} and \ul{underlined}, respectively.
}
\label{table:visa_pro}
\resizebox{0.9\linewidth}{!}{
\begin{tabular}{rcccccccc}
\toprule
\multicolumn{1}{c}{\multirow{2.5}{*}{Method}} 
& \multicolumn{2}{c}{\textbf{1 × 12 with 12 Steps}}  
& \multicolumn{2}{c}{\textbf{8 -- 1 × 4 with 5 Stepss}} 
& \multicolumn{2}{c}{\textbf{8 -- 4 with 2 Steps}} 
& \multicolumn{2}{c}{\textbf{11 -- 1 with 2 Steps}} 
\\ 
\cmidrule(r){2-3} \cmidrule(r){4-5} \cmidrule(r){6-7} \cmidrule(r){8-9} 
& \textbf{AUPRO ($\uparrow$)} & \textbf{FM ($\downarrow$)} & \textbf{AUPRO ($\uparrow$)} & \textbf{FM ($\downarrow$)} & \textbf{AUPRO ($\uparrow$)} & \textbf{FM ($\downarrow$)} & \textbf{AUPRO ($\uparrow$)}& \textbf{FM ($\downarrow$)}\\
\midrule
UCAD$\;_\text{AAAI'24}$ \cite{ucad} 
& 80.0 & \ul{0.5} 
& {70.5}& {9.8} 
& {72.4}& 7.5   
& {80.4}& 2.0   
\\
IUF$\;_\text{ECCV'24}$ \cite{iuf} 
& \ul{83.3} & 4.7 
& 57.0& 24.0 
& 63.9  & 20.8   
& \ul{82.0} & {1.0} 
\\
CDAD$\;_\text{CVPR'25}$ \cite{cdad}
& 59.0 & 17.0 
& \ul{77.7} & \textbf{1.7} 
& \ul{78.9} & \ul{1.6}  
& 81.6 & \textbf{-0.2} 
\\
\midrule
\rowcolor{red!25}
\ourso 
& \textbf{96.5} & \textbf{0.1}  
& \textbf{94.1} & \ul{1.8} 
& \textbf{96.7} & \textbf{0.4}
& \textbf{94.8} & \ul{0.6} 
\\ 
\bottomrule
\end{tabular}
}
\end{table*}

\paragraph{Additional quantitative results} We compare \ours with the SOTA baselines in terms of task mean image-level Average Precision (AP) and pixel-level Area Under the Per-Region-Overlap (AUPRO). We report the AP and AUPRO on the MVTecAD dataset in Tab.~\ref{table:mvtec_ap} and Tab.~\ref{table:mvtec_pro}, respectively. We also report the AP and AUPRO on the VisA dataset in  Tab.~\ref{table:visa_ap} and Tab.~\ref{table:visa_pro}, respectively. All the aforesaid tables show that \ours outperform the SOTA baselines in most cases.

\paragraph{Additional ablation results} We condcut additional ablation experiments with VisA dataset with results in Fig.~\ref{fig:ablation_visa} which shows similar aspects as the ablation study on MVTecAD dataset: larger coreset size showing better rsults, sampling ratio doesn't show critical effect, and moderately high approximate ratio saves time with negligible performance drop.

\paragraph{Comparing with different sampling methods} We compare our proposed sampling with Reservoir and Herding. Reservoir sampling is uniform (density, not coverage) and Herding matches the mean and needs class labels, absent in unsupervised CAD, so we adapt it with task means. Loss-based coreset selection is undefined without a training loss. Our contribution is making this coverage objective continuable under fixed memory and unknown $T$, with the additive-error guarantee of Theorem.~\ref{theorem}. The continuation is non-trivial beyond naive repeated greedy downsampling. Expansion conditions on the current coreset via the base set, so only $n_t$ coverage-adding features enter memory before consolidation, the two-step structure Theorem~4.1 analyzes. We ran the requested fixed-memory baselines with identical budget ($m{=}20$k), features, and inference shown in Tab.~\ref{table:reb_sampling}, averaged over all schedules.

\begin{table*}[t]
\centering
\caption{Full restuls of ablation on different backbones on MVTecAD dataset.}
\label{table:abl_bbone_mvtec_all}
\resizebox{0.9\linewidth}{!}{
\begin{tabular}{r@{\hspace{2mm}}c@{\hspace{1mm}}c@{\hspace{1mm}}c@{\hspace{1mm}}c@{\hspace{1mm}}c@{\hspace{1mm}}c@{\hspace{1mm}}c@{\hspace{1mm}}c@{\hspace{1mm}}c@{\hspace{1mm}}c}
\toprule
\multicolumn{1}{c}{\multirow{2.5}{*}{Backbone}} 
& \multicolumn{2}{c}{\textbf{1 × 15 with 15 Steps}} 
& \multicolumn{2}{c}{\textbf{3 × 5 with 5 Steps}} 
& \multicolumn{2}{c}{\textbf{10 -- 1 × 5 with 6 Steps}} 
& \multicolumn{2}{c}{\textbf{10 -- 5 with 2 Steps}} 
& \multicolumn{2}{c}{\textbf{14 -- 1 with 2 Steps}} 
\\ 
\cmidrule(r){2-3} \cmidrule(r){4-5} \cmidrule(r){6-7} \cmidrule(r){8-9} \cmidrule(r){10-11}
& \textbf{AUROC ($\uparrow$)} & \textbf{FM ($\downarrow$)} & \textbf{AUROC ($\uparrow$)} & \textbf{FM ($\downarrow$)} & \textbf{AUROC ($\uparrow$)} & \textbf{FM ($\downarrow$)} & \textbf{AUROC ($\uparrow$)}& \textbf{FM ($\downarrow$)} & \textbf{AUROC ($\uparrow$)}& \textbf{FM ($\downarrow$)}\\
\midrule
DenseNet201 & 98.1 / 96.8 & 0.8 / 0.3 & 95.7 / 95.1 & 1.7 / 1.0 & 98.4 / 95.7 & 0.5 / 0.2 & 95.5 / 93.3 & 1.2 / 0.6 & 97.0 / 96.1 & 0.0 / 0.0 \\ 
DINOv2 & 97.6 / 94.3 & 0.0 / 0.0 & 97.4 / 93.9 & -0.0 / 0.0 & 98.5 / 92.3 & 0.0 / 0.0 & 97.6 / 92.8 & 0.0 / 0.0 & 98.2 / 92.4 & 0.0 / 0.0 \\ 
EfficientNet b1 & 95.2 / 95.0 & 0.0 / 0.0 & 95.5 / 94.2 & 0.0 / -0.0 & 96.7 / 94.2 & 0.0 / 0.0 & 96.1 / 94.0 & 0.0 / -0.0 & 96.6 / 94.8 & 0.0 / 0.0 \\ 
WideResNet101 & 99.0 / 97.8 & 0.1 / 0.1 & 97.2 / 96.7 & 0.9 / 0.6 & 99.0 / 96.9 & 0.3 / 0.2 & 98.4 / 96.8 & 0.7 / 0.6 & 98.8 / 97.6 & 0.1 / 0.0 \\ 
\bottomrule
\end{tabular}
}
\end{table*}

\begin{table*}[t]
\centering
\caption{Full restuls of ablation on different backbones on VisA dataset.}
\label{table:abl_bbone_visa_all}
\resizebox{0.9\linewidth}{!}{
\begin{tabular}{rcccccccc}
\toprule
\multicolumn{1}{c}{\multirow{2.5}{*}{Backbone}} 
& \multicolumn{2}{c}{\textbf{1 × 12 with 12 Steps}}  
& \multicolumn{2}{c}{\textbf{8 -- 1 × 4 with 5 Stepss}} 
& \multicolumn{2}{c}{\textbf{8 -- 4 with 2 Steps}} 
& \multicolumn{2}{c}{\textbf{11 -- 1 with 2 Steps}} 
\\ 
\cmidrule(r){2-3} \cmidrule(r){4-5} \cmidrule(r){6-7} \cmidrule(r){8-9} 
& \textbf{AUROC ($\uparrow$)} & \textbf{FM ($\downarrow$)} & \textbf{AUROC ($\uparrow$)} & \textbf{FM ($\downarrow$)} & \textbf{AUROC ($\uparrow$)} & \textbf{FM ($\downarrow$)} & \textbf{AUROC ($\uparrow$)}& \textbf{FM ($\downarrow$)}\\
\midrule
DenseNet201 & 91.8 / 96.9 & 2.4 / 0.5 & 94.4 / 96.6 & 2.2 / 0.4 & 91.0 / 95.2 & 3.0 / 0.6 & 93.6 / 96.9 & 1.1 / 0.2 \\ 	
DINOv2 & 92.8 / 95.0 & 0.0 / 0.0 & 94.2 / 94.3 & 0.1 / 0.0 & 93.3 / 94.5 & -0.0 / -0.0 & 94.3 / 96.5 & 0.0 / -0.0 \\ 
EfficientNet b1 & 88.3 / 95.3 & 0.0 / 0.0 & 91.0 / 96.1 & 0.0 / 0.0 & 88.2 / 95.3 & 0.2 / -0.0 & 90.8 / 96.3 & 0.0 / -0.0 \\ 
WideResNet101 & 93.8 / 97.4 & 1.5 / 0.5 & 96.5 / 97.4 & 0.8 / 0.4 & 92.7 / 96.8 & 2.4 / 1.1 & 94.6 / 97.4 & -0.0 / 0.0 \\ 
\bottomrule
\end{tabular}
}
\end{table*}

\begin{table*}[t]
\centering
\caption{Full results of ablation on corset size $m$ on MVTecAD dataset.}
\label{table:abl_coreset_size_mvtec_all}
\resizebox{0.9\linewidth}{!}{
\begin{tabular}{r@{\hspace{2mm}}c@{\hspace{1mm}}c@{\hspace{1mm}}c@{\hspace{1mm}}c@{\hspace{1mm}}c@{\hspace{1mm}}c@{\hspace{1mm}}c@{\hspace{1mm}}c@{\hspace{1mm}}c@{\hspace{1mm}}c}
\toprule
\multicolumn{1}{c}{\multirow{2.5}{*}{Coreset Size $m$}} 
& \multicolumn{2}{c}{\textbf{1 × 15 with 15 Steps}} 
& \multicolumn{2}{c}{\textbf{3 × 5 with 5 Steps}} 
& \multicolumn{2}{c}{\textbf{10 -- 1 × 5 with 6 Steps}} 
& \multicolumn{2}{c}{\textbf{10 -- 5 with 2 Steps}} 
& \multicolumn{2}{c}{\textbf{14 -- 1 with 2 Steps}} 
\\ 
\cmidrule(r){2-3} \cmidrule(r){4-5} \cmidrule(r){6-7} \cmidrule(r){8-9} \cmidrule(r){10-11}
& \textbf{AUROC ($\uparrow$)} & \textbf{FM ($\downarrow$)} & \textbf{AUROC ($\uparrow$)} & \textbf{FM ($\downarrow$)} & \textbf{AUROC ($\uparrow$)} & \textbf{FM ($\downarrow$)} & \textbf{AUROC ($\uparrow$)}& \textbf{FM ($\downarrow$)} & \textbf{AUROC ($\uparrow$)}& \textbf{FM ($\downarrow$)}\\
\midrule
5,000 & 94.2 / 93.8 & 4.1 / 2.9 & 90.9 / 94.1 & 3.5 / 1.2 & 96.0 / 94.8 & 1.7 / 0.6 & 93.0 / 94.1 & 2.1 / 0.7 & 96.5 / 95.9 & 0.2 / 0.0 \\ 
10,000 & 94.7 / 96.2 & 4.3 / 1.5 & 95.1 / 95.6 & 1.8 / 0.7 & 98.5 / 95.8 & 0.5 / 0.2 & 97.1 / 95.7 & 0.8 / 1.0 & 97.7 / 96.4 & -0.1 / 0.0 \\ 
15,000 & 98.3 / 97.0 & 0.9 / 0.6 & 95.4 / 96.2 & 1.5 / 0.3 & 99.0 / 96.1 & 0.1 / 0.4 & 97.2 / 96.1 & 1.6 / 0.9 & 98.0 / 96.8 & 0.0 / 0.0 \\ 
20,000 & 98.8 / 97.4 & 0.2 / 0.2 & 98.2 / 96.9 & 0.6 / 0.5 & 98.9 / 96.4 & 0.3 / 0.2 & 97.7 / 96.6 & 0.5 / 0.5 & 98.3 / 97.3 & 0.0 / 0.1 \\ 
25,000 & 99.0 / 97.3 & 0.1 / 0.2 & 97.7 / 96.9 & 0.3 / 0.3 & 98.9 / 96.9 & 0.1 / 0.1 & 98.3 / 97.2 & 0.5 / 0.2 & 98.7 / 97.7 & -0.0 / 0.0 \\ 
\bottomrule
\end{tabular}
}
\end{table*}

\begin{table*}[t]
\centering
\caption{Full results of ablation on corset size $m$ on VisA dataset.}
\label{table:abl_coreset_size_visa_all}
\resizebox{0.9\linewidth}{!}{
\begin{tabular}{rcccccccc}
\toprule
\multicolumn{1}{c}{\multirow{2.5}{*}{Coreset Size $m$}} 
& \multicolumn{2}{c}{\textbf{1 × 12 with 12 Steps}}  
& \multicolumn{2}{c}{\textbf{8 -- 1 × 4 with 5 Stepss}} 
& \multicolumn{2}{c}{\textbf{8 -- 4 with 2 Steps}} 
& \multicolumn{2}{c}{\textbf{11 -- 1 with 2 Steps}} 
\\ 
\cmidrule(r){2-3} \cmidrule(r){4-5} \cmidrule(r){6-7} \cmidrule(r){8-9} 
& \textbf{AUROC ($\uparrow$)} & \textbf{FM ($\downarrow$)} & \textbf{AUROC ($\uparrow$)} & \textbf{FM ($\downarrow$)} & \textbf{AUROC ($\uparrow$)} & \textbf{FM ($\downarrow$)} & \textbf{AUROC ($\uparrow$)}& \textbf{FM ($\downarrow$)}\\
\midrule
10,000 & 86.5 / 94.5 & 6.9 / 2.1 & 93.2 / 95.3 & 3.2 / 1.3 & 89.0 / 94.5 & 3.4 / 1.0 & 92.3 / 96.2 & 0.9 / -0.2 \\ 
20,000 & 90.6 / 96.1 & 4.3 / 1.6 & 95.0 / 96.2 & 1.3 / 1.0 & 92.3 / 95.9 & 1.8 / 1.0 & 94.8 / 96.8 & 0.2 / -0.2 \\ 
30,000 & 92.6 / 96.8 & 2.1 / 0.9 & 95.2 / 96.6 & 2.1 / 0.7 & 93.3 / 96.4 & 1.4 / 0.6 & 94.8 / 97.1 & -0.2 / 0.1 \\ 
40,000 & 93.7 / 97.2 & 0.7 / 0.6 & 96.3 / 97.2 & 0.7 / 0.5 & 93.8 / 96.5 & 1.2 / 0.8 & 96.1 / 97.7 & 0.4 / -0.0 \\ 
50,000 & 94.0 / 97.5 & 0.4 / 0.2 & 97.0 / 97.4 & 0.5 / 0.4 & 94.6 / 97.2 & 0.5 / 0.2 & 96.2 / 97.8 & 0.4 / -0.0 \\ 
\bottomrule
\end{tabular}
}
\end{table*}

\begin{table*}[t]
\centering
\caption{Full results of ablation on sampling ratio $p$ on MVTecAD dataset.}
\label{table:abl_sampling_ratio_mvtec_all}
\resizebox{0.9\linewidth}{!}{
\begin{tabular}{r@{\hspace{2mm}}c@{\hspace{1mm}}c@{\hspace{1mm}}c@{\hspace{1mm}}c@{\hspace{1mm}}c@{\hspace{1mm}}c@{\hspace{1mm}}c@{\hspace{1mm}}c@{\hspace{1mm}}c@{\hspace{1mm}}c}
\toprule
\multicolumn{1}{c}{\multirow{2.5}{*}{Sampling Ratio $p$}} 
& \multicolumn{2}{c}{\textbf{1 × 15 with 15 Steps}} 
& \multicolumn{2}{c}{\textbf{3 × 5 with 5 Steps}} 
& \multicolumn{2}{c}{\textbf{10 -- 1 × 5 with 6 Steps}} 
& \multicolumn{2}{c}{\textbf{10 -- 5 with 2 Steps}} 
& \multicolumn{2}{c}{\textbf{14 -- 1 with 2 Steps}} 
\\ 
\cmidrule(r){2-3} \cmidrule(r){4-5} \cmidrule(r){6-7} \cmidrule(r){8-9} \cmidrule(r){10-11}
& \textbf{AUROC ($\uparrow$)} & \textbf{FM ($\downarrow$)} & \textbf{AUROC ($\uparrow$)} & \textbf{FM ($\downarrow$)} & \textbf{AUROC ($\uparrow$)} & \textbf{FM ($\downarrow$)} & \textbf{AUROC ($\uparrow$)}& \textbf{FM ($\downarrow$)} & \textbf{AUROC ($\uparrow$)}& \textbf{FM ($\downarrow$)}\\
\midrule
0.01 & 98.8 / 97.4 & 0.2 / 0.2 & 98.2 / 96.9 & 0.6 / 0.5 & 98.9 / 96.4 & 0.3 / 0.2 & 97.7 / 96.6 & 0.5 / 0.5 & 98.3 / 97.3 & 0.0 / 0.1 \\ 
0.05 & 97.5 / 96.6 & 1.4 / 1.3 & 96.6 / 96.7 & 1.6 / 0.7 & 98.4 / 96.7 & 0.9 / 0.3 & 97.8 / 96.7 & 0.4 / 0.8 & 97.5 / 97.1 & 0.1 / 0.0 \\ 
0.1 & 97.1 / 96.8 & 1.7 / 1.1 & 96.6 / 96.7 & 1.6 / 0.7 & 98.8 / 96.5 & 0.5 / 0.4 & 97.8 / 96.7 & 0.4 / 0.8 & 97.5 / 97.1 & 0.1 / 0.0 \\ 
0.25 & 97.4 / 96.6 & 1.5 / 1.4 & 96.6 / 96.7 & 1.6 / 0.7 & 98.9 / 96.6 & 0.5 / 0.3 & 97.8 / 96.7 & 0.4 / 0.8 & 97.6 / 97.1 & 0.0 / -0.0 \\ 
0.5 & 97.5 / 96.6 & 1.4 / 1.3 & 96.6 / 96.7 & 1.6 / 0.7 & 98.9 / 96.7 & 0.5 / 0.3 & 97.8 / 96.7 & 0.4 / 0.8 & 97.6 / 97.1 & 0.0 / -0.0 \\ 
\bottomrule
\end{tabular}
}
\end{table*}

\begin{table*}[t]
\centering
\caption{Full results of ablation on sampling ratio on VisA dataset.}
\label{table:abl_sampling_ratio_visa_all}
\resizebox{0.9\linewidth}{!}{
\begin{tabular}{rcccccccc}
\toprule
\multicolumn{1}{c}{\multirow{2.5}{*}{Sampling Ratio $p$}} 
& \multicolumn{2}{c}{\textbf{1 × 12 with 12 Steps}}  
& \multicolumn{2}{c}{\textbf{8 -- 1 × 4 with 5 Stepss}} 
& \multicolumn{2}{c}{\textbf{8 -- 4 with 2 Steps}} 
& \multicolumn{2}{c}{\textbf{11 -- 1 with 2 Steps}} 
\\ 
\cmidrule(r){2-3} \cmidrule(r){4-5} \cmidrule(r){6-7} \cmidrule(r){8-9} 
& \textbf{AUROC ($\uparrow$)} & \textbf{FM ($\downarrow$)} & \textbf{AUROC ($\uparrow$)} & \textbf{FM ($\downarrow$)} & \textbf{AUROC ($\uparrow$)} & \textbf{FM ($\downarrow$)} & \textbf{AUROC ($\uparrow$)}& \textbf{FM ($\downarrow$)}\\
\midrule
0.01 & 93.7 / 97.2 & 0.7 / 0.6 & 96.3 / 97.2 & 0.7 / 0.5 & 93.8 / 96.5 & 1.2 / 0.8 & 96.1 / 97.7 & 0.4 / -0.0 \\ 
0.05 & 92.4 / 96.6 & 2.7 / 1.1 & 96.1 / 97.3 & 1.1 / 0.6 & 93.9 / 97.0 & 1.1 / 0.6 & 95.5 / 97.4 & 0.1 / -0.1 \\ 
0.1 & 91.8 / 96.1 & 3.2 / 1.5 & 95.9 / 97.4 & 1.7 / 0.5 & 93.9 / 97.0 & 1.1 / 0.6 & 95.3 / 97.3 & -0.2 / -0.0 \\ 
0.25 & 91.8 / 95.7 & 3.2 / 2.0 & 95.9 / 97.3 & 1.7 / 0.5 & 93.9 / 97.0 & 1.1 / 0.6 & 95.3 / 97.4 & -0.2 / -0.0 \\ 
0.5 & 91.8 / 95.7 & 3.2 / 2.0 & 95.9 / 97.3 & 1.7 / 0.5 & 93.9 / 97.0 & 1.1 / 0.6 & 95.3 / 97.4 & -0.2 / -0.0 \\  
\bottomrule
\end{tabular}
}
\end{table*}

\begin{table*}[t]
\centering
\caption{Full results of ablation on online continual on MVTecAD dataset.}
\label{table:abl_online_continual_mvtec_all}
\resizebox{0.9\linewidth}{!}{
\begin{tabular}{r@{\hspace{2mm}}c@{\hspace{1mm}}c@{\hspace{1mm}}c@{\hspace{1mm}}c@{\hspace{1mm}}c@{\hspace{1mm}}c@{\hspace{1mm}}c@{\hspace{1mm}}c@{\hspace{1mm}}c@{\hspace{1mm}}c}
\toprule
\multicolumn{1}{c}{\multirow{2.5}{*}{Method}} 
& \multicolumn{2}{c}{\textbf{1 × 15 with 15 Steps}} 
& \multicolumn{2}{c}{\textbf{3 × 5 with 5 Steps}} 
& \multicolumn{2}{c}{\textbf{10 -- 1 × 5 with 6 Steps}} 
& \multicolumn{2}{c}{\textbf{10 -- 5 with 2 Steps}} 
& \multicolumn{2}{c}{\textbf{14 -- 1 with 2 Steps}} 
\\ 
\cmidrule(r){2-3} \cmidrule(r){4-5} \cmidrule(r){6-7} \cmidrule(r){8-9} \cmidrule(r){10-11}
& \textbf{AUROC ($\uparrow$)} & \textbf{FM ($\downarrow$)} & \textbf{AUROC ($\uparrow$)} & \textbf{FM ($\downarrow$)} & \textbf{AUROC ($\uparrow$)} & \textbf{FM ($\downarrow$)} & \textbf{AUROC ($\uparrow$)}& \textbf{FM ($\downarrow$)} & \textbf{AUROC ($\uparrow$)}& \textbf{FM ($\downarrow$)}\\
\midrule
DNE$\;_\text{ACMMM'22}$ \cite{dne} & 76.3 / - & 2.4 / - & 70.9 / - & 2.6 / - & 66.2 / - & 1.0 / - & 67.2 / - & 0.3 / - & 64.7 / - & 0.2 / - \\ 
UCAD$\;_\text{AAAI'24}$ \cite{ucad} & 91.0 / 83.3 & 0.0 / 0.8 & 81.2 / 75.3 & 0.1 / 1.3 & 89.9 / 74.7 & 0.7 / 1.5 & 78.6 / 77.3 & -1.0 / 3.5 & 80.3 / 70.6 & -5.3 / -5.4 \\ 
IUF$\;_\text{ECCV'24}$ \cite{iuf} & 69.4 / 78.2 & 0.1 / 0.2 & 71.0 / 80.3 & 0.4 / 0.1 & 71.9 / 78.9 & 0.8 / 0.1 & 72.1 / 81.7 & 0.0 / 0.1 & 70.8 / 83.4 & 0.5 / -0.2 \\ 
CDAD$\;_\text{CVPR'25}$ \cite{cdad} & 67.8 / 76.0 & 2.7 / 0.0 & 60.4 / 77.0 & 3.7 / 0.2 & 62.4 / 77.5 & 4.7 / 0.0 & 70.5 / 83.1 & -0.9 / 0.1 & 69.9 / 88.2 & 0.0 / 0.0 \\ 
\midrule
\rowcolor{red!25}
\textbf{ContCore} & \textbf{96.7 / 96.7} & \textbf{2.1 / 1.2} & \textbf{96.3 / 96.4} & \textbf{2.5 / 1.3} & \textbf{98.4 / 96.4} & \textbf{0.6 / 0.4} & \textbf{96.8 / 96.5} & \textbf{2.3 / 1.2} & \textbf{98.0 / 97.5} & \textbf{-0.1 / 0.2} \\ 
\bottomrule
\end{tabular}
}
\end{table*}

\begin{table*}[t]
\centering
\caption{Full results of ablation on online continual on VisA dataset.}
\label{table:abl_online_continual_visa_all}
\resizebox{0.9\linewidth}{!}{
\begin{tabular}{rcccccccc}
\toprule
\multicolumn{1}{c}{\multirow{2.5}{*}{Method}} 
& \multicolumn{2}{c}{\textbf{1 × 12 with 12 Steps}}  
& \multicolumn{2}{c}{\textbf{8 -- 1 × 4 with 5 Stepss}} 
& \multicolumn{2}{c}{\textbf{8 -- 4 with 2 Steps}} 
& \multicolumn{2}{c}{\textbf{11 -- 1 with 2 Steps}} 
\\ 
\cmidrule(r){2-3} \cmidrule(r){4-5} \cmidrule(r){6-7} \cmidrule(r){8-9} 
& \textbf{AUROC ($\uparrow$)} & \textbf{FM ($\downarrow$)} & \textbf{AUROC ($\uparrow$)} & \textbf{FM ($\downarrow$)} & \textbf{AUROC ($\uparrow$)} & \textbf{FM ($\downarrow$)} & \textbf{AUROC ($\uparrow$)}& \textbf{FM ($\downarrow$)}\\
\midrule
DNE$\;_\text{ACMMM'22}$ \cite{dne} & 61.7 / 0.0 & 3.7 / 0.0 & 62.8 / 0.0 & 3.2 / 0.0 & 65.6 / 0.0 & 5.7 / 0.0 & 69.7 / 0.0 & -0.3 / 0.0 \\ 
UCAD$\;_\text{AAAI'24}$ \cite{ucad}  & 78.3 / 76.9 & 0.0 / 0.0 & 75.4 / 79.8 & -0.2 / 0.5 & 66.7 / 73.1 & -0.2 / 0.2 & 78.0 / 73.3 & 0.7 / 1.0 \\ 
IUF$\;_\text{ECCV'24}$ \cite{iuf} & 58.0 / 80.4 & 0.2 / -0.0 & 57.0 / 84.8 & 6.0 / 2.7 & 64.9 / 87.1 & -1.4 / 0.9 & 61.4 / 87.7 & -1.2 / 3.2 \\ 
CDAD$\;_\text{CVPR'25}$ \cite{cdad}  & 57.3 / 83.5 & 0.1 / 0.0 & 56.3 / 86.2 & 1.3 / 0.6 & 60.0 / 86.9 & -1.6 / -0.5 & 61.9 / 93.3 & 0.9 / 0.9 \\ 
\midrule
\rowcolor{red!25}
 \textbf{ContCore} & \textbf{91.8 / 96.6} & \textbf{2.3 / 1.1} & \textbf{96.2 / 97.4} & \textbf{0.3 / 0.6} & \textbf{93.1 / 97.0} & \textbf{1.8 / 1.0} & \textbf{95.1 / 97.8} & \textbf{0.8 / 0.3} \\ 
\bottomrule
\end{tabular}
}
\end{table*}

\begin{table*}[t]
\centering
\caption{Full results of ablation on approximate ratio on MVTecAD dataset.}
\label{table:abl_approximate_ratio_mvtec_all}
\resizebox{\linewidth}{!}{
\begin{tabular}{r@{\hspace{2mm}}c@{\hspace{1mm}}c@{\hspace{1mm}}c@{\hspace{1mm}}c@{\hspace{1mm}}c@{\hspace{1mm}}c@{\hspace{1mm}}c@{\hspace{1mm}}c@{\hspace{1mm}}c@{\hspace{1mm}}c@{\hspace{1mm}}c@{\hspace{1mm}}c@{\hspace{1mm}}c@{\hspace{1mm}}c@{\hspace{1mm}}c}
\toprule
\multicolumn{1}{c}{\multirow{2.5}{*}{Approximate Ratio $q$}} 
& \multicolumn{3}{c}{\textbf{1 × 15 with 15 Steps}} 
& \multicolumn{3}{c}{\textbf{3 × 5 with 5 Steps}} 
& \multicolumn{3}{c}{\textbf{10 -- 1 × 5 with 6 Steps}} 
& \multicolumn{3}{c}{\textbf{10 -- 5 with 2 Steps}} 
& \multicolumn{3}{c}{\textbf{14 -- 1 with 2 Steps}} 
\\ 
\cmidrule(r){2-4} \cmidrule(r){5-7} \cmidrule(r){8-10} \cmidrule(r){11-13} \cmidrule(r){14-16}
& \textbf{AUROC ($\uparrow$)} & \textbf{FM ($\downarrow$)} & \textbf{Time ($\downarrow$)} 
& \textbf{AUROC ($\uparrow$)} & \textbf{FM ($\downarrow$)} & \textbf{Time ($\downarrow$)} 
& \textbf{AUROC ($\uparrow$)} & \textbf{FM ($\downarrow$)} & \textbf{Time ($\downarrow$)}
& \textbf{AUROC ($\uparrow$)}& \textbf{FM ($\downarrow$)} & \textbf{Time ($\downarrow$)}
& \textbf{AUROC ($\uparrow$)}& \textbf{FM ($\downarrow$)} & \textbf{Time ($\downarrow$)} \\
\midrule
0 & 98.8 / 97.2 & 0.3 / 0.2 & 95 & 97.2 / 96.7 & 0.5 / 0.3 & 131 & 98.7 / 96.6 & 0.3 / 0.2 & 163 & 97.9 / 96.6 & 1.3 / 0.6 & 172 & 98.4 / 97.5 & -0.0 / 0.0 & 173 \\ 
0.25 & 98.8 / 97.3 & 0.2 / 0.2 & 90 & 97.3 / 96.7 & 0.5 / 0.3 & 129 & 98.8 / 96.5 & 0.4 / 0.2 & 158 & 97.2 / 96.7 & 1.9 / 0.7 & 169 & 98.0 / 97.4 & 0.2 / -0.0 & 171 \\ 
0.5 & 98.3 / 97.2 & 0.8 / 0.2 & 83 & 97.0 / 96.7 & 0.6 / 0.4 & 127 & 98.6 / 96.3 & 0.5 / 0.3 & 151 & 97.5 / 96.7 & 1.5 / 0.6 & 167 & 97.8 / 97.5 & 0.1 / 0.0 & 168 \\ 
0.75 & 98.8 / 97.4 & 0.2 / 0.2 & 79 & 98.2 / 96.9 & 0.6 / 0.5 & 126 & 98.9 / 96.4 & 0.3 / 0.2 & 147 & 97.7 / 96.6 & 0.5 / 0.5 & 165 & 98.3 / 97.3 & 0.0 / 0.1 & 167 \\ 
1 & 96.4 / 96.6 & 0.8 / 0.3 & 75 & 95.1 / 95.2 & 0.7 / 0.4 & 124 & 86.9 / 91.0 & 1.2 / 0.3 & 142 & 94.5 / 94.1 & 1.8 / 1.1 & 164 & 72.6 / 91.7 & 0.7 / 0.0 & 165 \\ 
\bottomrule
\end{tabular}
}
\end{table*}

\begin{table*}[t]
\centering
\caption{Full results of ablation on approximate ratio on VisA dataset.}
\label{table:abl_approximate_ratio_visa_all}
\resizebox{\linewidth}{!}{
\begin{tabular}{rcccccccccccc}
\toprule
\multicolumn{1}{c}{\multirow{2.5}{*}{Approximate Ratio $q$}} 
& \multicolumn{3}{c}{\textbf{1 × 12 with 12 Steps}}  
& \multicolumn{3}{c}{\textbf{8 -- 1 × 4 with 5 Stepss}} 
& \multicolumn{3}{c}{\textbf{8 -- 4 with 2 Steps}} 
& \multicolumn{3}{c}{\textbf{11 -- 1 with 2 Steps}} 
\\ 
\cmidrule(r){2-4} \cmidrule(r){5-7} \cmidrule(r){8-10} \cmidrule(r){11-13} 
& \textbf{AUROC ($\uparrow$)} & \textbf{FM ($\downarrow$)} & \textbf{Time ($\downarrow$)} 
& \textbf{AUROC ($\uparrow$)} & \textbf{FM ($\downarrow$)} & \textbf{Time ($\downarrow$)} 
& \textbf{AUROC ($\uparrow$)} & \textbf{FM ($\downarrow$)} & \textbf{Time ($\downarrow$)}
& \textbf{AUROC ($\uparrow$)}& \textbf{FM ($\downarrow$)} & \textbf{Time ($\downarrow$)} \\
\midrule
0 & 93.3 / 97.2 & 1.2 / 0.5 & 372 & 96.2 / 97.2 & 1.2 / 0.5 & 452 & 94.0 / 97.1 & 2.1 / 0.5 & 458 & 95.7 / 97.5 & -0.1 / -0.1 & 458 \\ 
0.25 & 93.2 / 97.3 & 1.2 / 0.4 & 352 & 96.6 / 97.5 & 1.1 / 0.4 & 432 & 94.3 / 97.2 & 1.5 / 0.2 & 442 & 96.0 / 97.5 & -0.2 / -0.0 & 442 \\ 
0.5 & 93.4 / 97.2 & 1.0 / 0.5 & 333 & 96.7 / 97.3 & 0.8 / 0.5 & 412 & 94.0 / 97.2 & 1.9 / 0.2 & 426 & 96.0 / 97.6 & 0.4 / -0.1 & 427 \\ 
0.75 & 93.7 / 97.2 & 0.7 / 0.6 & 319 & 96.3 / 97.2 & 0.7 / 0.5 & 399 & 93.8 / 96.5 & 1.2 / 0.8 & 414 & 96.1 / 97.7 & 0.4 / -0.0 & 416 \\ 
1 & 90.7 / 96.7 & 3.8 / 1.0 & 311 & 87.1 / 95.8 & 0.9 / 0.5 & 391 & 84.9 / 94.3 & 2.1 / 0.6 & 408 & 91.8 / 96.1 & 2.8 / 0.0 & 408 \\ 
\bottomrule
\end{tabular}
}
\end{table*}

\begin{table}[t]
\centering
\caption{Averaged results of MVTecAD dataset with 5 schedules with different sampling methods with (image/pixel) format.}
\label{table:reb_sampling}
\small
\setlength{\tabcolsep}{15pt}
\resizebox{0.3\linewidth}{!}{
\begin{tabular}{rcc}
\toprule
Method & Task AVG & FM \\
\midrule
Reservoir & 92.2 / 96.5 & 1.4 / 0.3 \\
Herding   & 90.1 / 96.3 & 2.0 / 0.6 \\
\midrule
\rowcolor{red!25}
Ours      & \textbf{98.4 / 96.9} & \textbf{0.3 / 0.3}\\
\bottomrule
\end{tabular}
}
\end{table}

\paragraph{Qualitative results}
\label{sec:segmentation}
We compare the quality of pixel-wise anomaly detection of the SOTA baselines with \ours in Fig.~\ref{fig:segmentation_1_1_14}, IUF \cite{iuf} and CDAD \cite{cdad} show poor results presenting false positive segmentation in normal areas. UCAD \cite{ucad} shows less than IUF and CDAD, yet still produces false positives. As shown in the samples from schedules 10 -- 5 (Fig.~\ref{fig:segmentation_10_5_1}), 10 -- 1$\times$5 (Fig.~\ref{fig:segmentation_10_1_5}), and 14 -- 1 (Fig.~\ref{fig:segmentation_14_1_1}), UCAD shows poor anomaly pixel detection with tasks with multiple classes, while IUF and CDAD show poor results when the number of additional tasks increases for multiple steps as shown in Fig.~\ref{fig:segmentation_1_1_14}, Fig.~\ref{fig:segmentation_10_1_5}, and Fig.~\ref{fig:segmentation_3_3_4}, suffering from catastrophic forgetting. Our propose method \ours, on the other hand, shows the least false positives and predicts the anomaly maps closest to the ground truth comparing with the SOTA baselines.

\begin{figure*}[t]
\centering
\includegraphics[width=\linewidth]{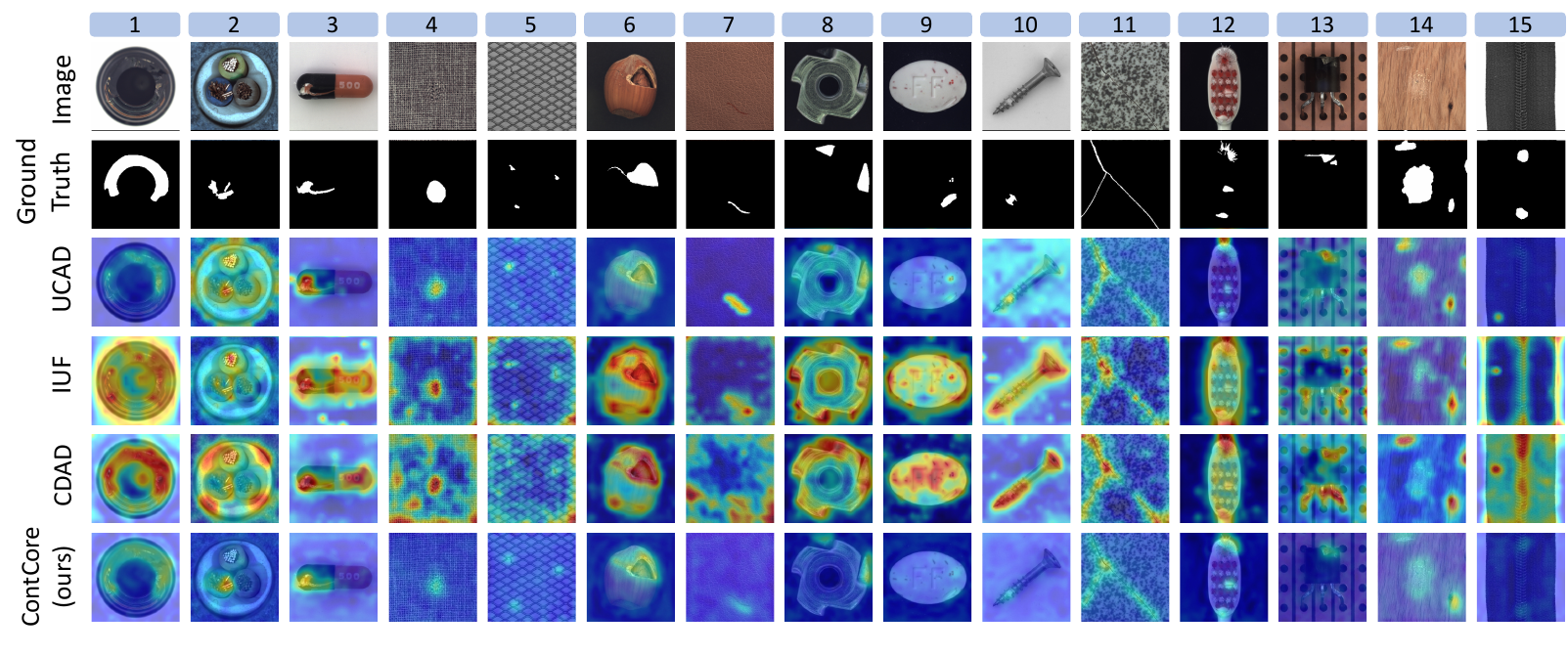}
\caption{
The qualitative results with $1\times15$ task schedule on the MVTecAD dataset.
}
\label{fig:segmentation_1_1_14}
\end{figure*}

\begin{figure*}[t]
\centering
\includegraphics[width=\linewidth]{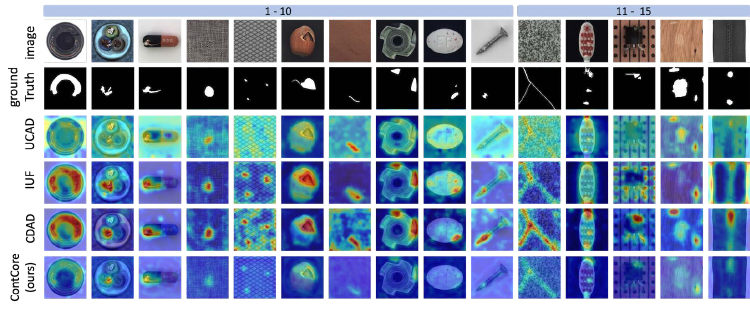}
\caption{
The qualitative results with 10 -- 5 task schedule on the MVTecAD dataset.
}
\label{fig:segmentation_10_5_1}
\end{figure*}

\begin{figure*}[t]
\centering
\includegraphics[width=\linewidth]{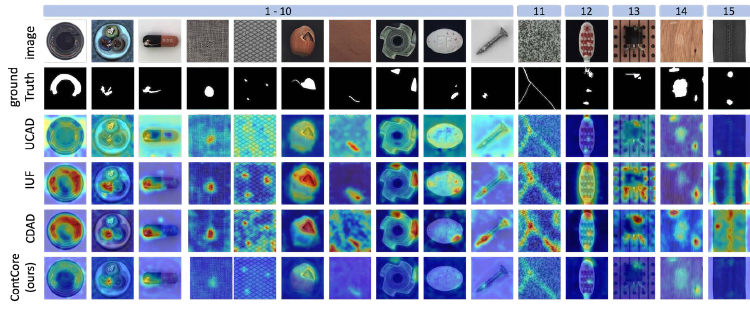}
\caption{
The qualitative results with 10 -- 1$\times$5 task schedule on the MVTecAD dataset.
}
\label{fig:segmentation_10_1_5}
\end{figure*}

\begin{figure*}[t]
\centering
\includegraphics[width=\linewidth]{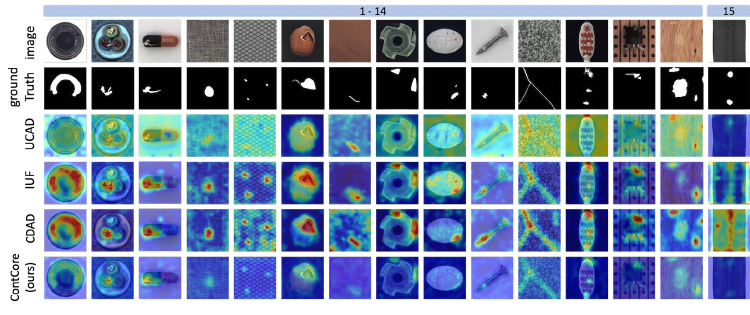}
\caption{
The qualitative results with 14 -- 1 task schedule on the MVTecAD dataset.
}
\label{fig:segmentation_14_1_1}
\end{figure*}

\begin{figure*}[t]
\centering
\includegraphics[width=\linewidth]{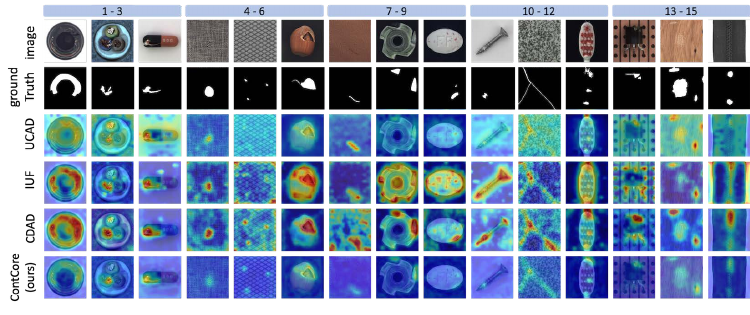}
\caption{
The qualitative results with 3$\times$5 task schedule on the MVTecAD dataset.
}
\label{fig:segmentation_3_3_4}
\end{figure*}

\end{document}